%% file: egpaper_final.tex
\documentclass[10pt,twocolumn,letterpaper]{article}

\usepackage{cvpr}
\usepackage{times}
\usepackage{epsfig}
\usepackage{graphicx}
\usepackage{amsmath}
\usepackage{amssymb}

\usepackage{subcaption}
\usepackage{makecell}
\usepackage{hhline}
\usepackage{multirow}
\usepackage[shortlabels]{enumitem}
\usepackage{array}

\usepackage{xspace}

\newcommand{\ours}{ODST\xspace}
\newcommand{\base}{ST\xspace}
\newcommand{\baseod}{ST-OT\xspace}

\def\R{\mathbb{R}}

\newcommand{\f}[1]{f^{(#1)}}

\def\T{\mathbf{T}}
\def\U{\mathbf{U}}
\def\I{\mathbf{I}}

\def\Exp{\mathbb{E}}

\def\ones{\mathop{\rm 1}\nolimits}

\def\pin{\mathrm{p_{\textrm{in}}}}

\def\pall{\mathrm{p_{\textrm{all}}}}

\def\max{\mathop{\rm max}\nolimits}
\def\ones{\mathbf{1}}

\newtheorem{lemma}{Lemma}[section]

\newenvironment{proof}{\par\noindent{\bf Proof:\ }}{\hfill$\Box$\\[2mm]}

\graphicspath{ {./images/} }

\usepackage[breaklinks=true,bookmarks=false]{hyperref}

\cvprfinalcopy 

\def\cvprPaperID{8644} 

\begin{document}

\title{
Out-distribution aware Self-training in an Open World Setting}

\author{Maximilian Augustin\\
University of T\"ubingen\\
{\tt\small  maximilian.augustin@uni-tuebingen.de}
\and
Matthias Hein\\
University of T\"ubingen\\
{\tt\small matthias.hein@uni-tuebingen.de}
}

\maketitle

\begin{abstract}

Deep Learning heavily depends on large labeled datasets which limits further improvements. While unlabeled data is available in large amounts, in particular in image recognition, it does not fulfill the 
closed world assumption of semi-supervised learning 
that all unlabeled data are task-related.
The goal of this paper is to leverage
unlabeled data in an open world setting to further improve prediction performance.
For this purpose, we introduce out-distribution aware self-training, which includes a careful sample selection strategy based on the confidence of the classifier. 
While normal self-training deteriorates prediction performance, our iterative scheme improves using up to 15 times the amount of originally labeled data.
Moreover, our classifiers are by design out-distribution aware and can thus distinguish task-related inputs from unrelated ones. 
\end{abstract}

\section{Introduction} 

In past years we have seen tremendous progress in image recognition based on deep learning \cite{KriSutHin2012,he2016deep,tan2020efficientnet}. However, this success also required very large labeled datasets that are expensive to generate. On the other hand, large amounts of unlabeled data are broadly available, in particular in image recognition. The promise of semi-supervised learning \cite{Chapelle2006SSL} is to leverage unlabeled data in order to improve prediction performance compared to only using labeled data. However, the underlying assumption of most
semi-supervised learning (SSL) algorithms \cite{Chapelle2006SSL}  is that the unlabeled data comes from the same distribution or at least contains data from the same set of classes as the labeled data (closed world assumption). 
\cite{oliver2018realistic} criticized this as being unrealistic since the assumption is hard to control when retrieving large
amounts of unlabeled data from the web
\input{figures/intro_opener}
SSL in an open world setting where the unlabeled data contains task-relevant but also non-related images has recently attracted attention \cite{athiwaratkun2019consistent,guo2020self,ChenEtAL2020SSL} as a more realistic approach to SSL. However, these approaches either work in a setting where the number of labeled samples and the ratio of non-relevant to task-relevant unlabeled images is quite small or they use weak supervision. The goal of this paper is to show that one can leverage large amounts of unlabeled data (80 million tiny images) for medium-sized labeled datasets (CIFAR10/100) in order to improve prediction performance even when the ratio of non-relevant to task-relevant images is huge (80 million tiny images contains roughly 100 times more samples from unrelated classes than from the ones of CIFAR10). Closest in spirit are the self-training approaches \cite{yalniz2019billion,xie2020self} which use large amounts of unlabeled data to improve the performance on ImageNet. While they are also working in an open world setting, the ratio of non-relevant to task-relevant images is much smaller due to a large number of classes. They thus suffer less from overconfident predictions of deep neural networks on out-distribution images \cite{NguYosClu2015, hendrycks2016baseline, HeiAndBit2019} which deteriorates the sample selection process in self-training as illustrated in Figure~\ref{fig:teaser}. 
These mistakes in the labeling process accumulate, create a distribution shift and can degrade prediction performance or at least prevent further improvements. 

In this paper, we propose an out-distribution aware self-training approach which we run iteratively with increasing amounts of pseudo-labeled unlabeled data. By enforcing low confidence predictions on out-distribution images in combination with carefully designed class-specific sample selection and calibration strategies, our pseudo-labeled images are of high-quality, see Figure \ref{fig:teaser} and \ref{fig:cifar10_sample_demonstration}. Thus we can train with large amounts of pseudo-labeled images (up to 15 times more than the size of the labeled set) even with a huge ratio of non-task-related to task-related images. We always improve the base classifier trained on the labeled set and in some cases even improve over three iterations of self-training while a standard self-training approach degrades much earlier and can sometimes fail to even surpass the base classifier.
On CIFAR10 and CIFAR100, we achieve the best-known results for our employed ResNet50 and PyramidNet272 architectures. Additionally, our classifiers have excellent out-distribution detection performance and thus know when they don't know.

\section{Related Work}
\textbf{Semi-supervised learning (SSL)} is an established technique \cite{zhu2005semi, zhu2009introduction}  for leveraging information from unlabeled data to improve predictive performance.
In self-training \cite{riloff1996automatically, riloff2003learning, scudder1965probability}, a teacher model is trained in a fully-supervised fashion on a labeled dataset. The teacher model is then used to label a set of unlabeled examples, typically drawn from the original data distribution (closed world assumption), which is then used in combination with the labeled samples to train a new student model.
Various extensions of this protocol have been proposed, including the use of an ensemble of teacher models \cite{zhou2018edf} and Co-training \cite{blum1998combining}.

Recently, self-training has been used to improve performance on ImageNet \cite{russakovsky2015imagenet}, by 
using large image databases consisting of millions of task-relevant and out-of-distribution samples \cite{xie2020self,yalniz2019billion}. 
On CIFAR10,  \cite{CarEtAl19} were able to significantly improve model robustness to adversarial perturbations by adding unlabeled samples from 80 million tiny images (80MTI) \cite{torralba200880} using self-training.
The distinctive feature of self-training in comparison to other SSL methods is that the training of the teacher model is separated from the labeling process. In contrast, in pseudo-labeling \cite{lee2013pseudo, iscen2019label, shi2018transductive} labels are generated during training by the model itself. Similarly, consistency-based SSL-methods like $\Pi$-models \cite{laine2016temporal, sajjadi2016regularization}, mean-teacher \cite{tarvainen2017mean} and virtual adversarial training \cite{miyato2018virtual} enforce an invariance of the model's output on the unlabeled data under a specific set of perturbations. A related technique is entropy-minimization \cite{grandvalet2005semi}, which \emph{penalizes} low-confidence predictions on unlabeled samples during training. \cite{oliver2018realistic} found that SSL can improve the model's performance in the traditional SSL setting where the unlabeled data is sampled from the same distribution as the training data (closed world assumption) but can degrade the performance when the unlabeled data contains non-task-related samples (open world setting).  


The open world setting has recently been tackled in 
\cite{athiwaratkun2019consistent} for a 500k subset of 80MTI selected using weak labels with established SSL methods together with stochastic weight averaging. On CIFAR100 they improve by less than $0.6\%$ or even degrade performance compared to just using the labeled training set of CIFAR100.
In \cite{guo2020self,ChenEtAL2020SSL,yu2020multitask} the open world setting is considered but either work only on small label size problems or suffer from severe performance degradation when the number of non-task-related unlabeled samples exceeds the number of labeled examples. In contrast,  for CIFAR10 we use the full 80MTI dataset as unlabeled data resulting in roughly 100 times more unrelated to task-specific images and \emph{improve} prediction performance.


\textbf{Out-of-distribution detection (OOD):} Deep Neural networks (DNN) have empirically and theoretically been shown to produce overconfident predictions for inputs not related to the task e.g. noise or other classes \cite{NguYosClu2015, hendrycks2016baseline, HeiAndBit2019}, \ie the confidence of a DNN is not reliable for the detection of out-of-distribution samples. Approaches for OOD include ODIN \cite{LiaLiSri2018} or using the Mahalanobis distance of higher-order features \cite{lee2018simple}. \cite{HenMazDie2019} introduces Outlier exposure (OE), see \cite{HeiAndBit2019} for the related CEDA, and show that the confidence can be used as a  reliable OOD-detector when enforcing low confidence on 80MTI during training, even when tested on out-distribution datasets other than 80MTI. OOD detection is related to open set recognition \cite{BoultEtAL2019OpenWorld} which goes beyond the scope of this paper. Our main focus is a better classifier, not only in terms of test accuracy but also in terms of OOD detection and robustness to image corruptions \cite{hendrycks2019benchmarking}. 


\section{Method}

We introduce self-training as in \cite{xie2020self} and then highlight the differences to our out-distribution aware self-training.

\subsection{Self-training}
Let $\T=(x_i,y_i)_{i=1}^{n}$ be our set of labeled examples,  where $x_i \in \R^d$ and $y_i \in \{1,\ldots,K\}$, and we assume to be given a collection of unlabeled samples $\U=(z_i)_{i=1}^{m}$. Traditional SSL literature makes the assumption that the unlabeled samples $\U$ are drawn from the same distribution as the labeled examples $\T$, or at least belong to the same set of classes. Given a neural network $f:\R^d \rightarrow \R$
the predicted probability distribution for a point $x$ is the softmax:
\[ \hat{p}_f(s|x)=\frac{e^{f_s(x)}}{\sum_{l=1}^K e^{f_l(x)}}. \]
The confidence in the decision for $x$ is then given by $\max_{s=1,\ldots,K} \hat{p}_f(s|x)$,
and the cross-entropy loss between (soft)-labels $p \in \R^K$ ($\sum_i p_i=1$, $p_i\geq 0$) and prediction $\hat{p}$ is defined as:
\[ L(p,\hat{p}) = - \sum_{i=1}^K p_i \log \hat{p}_i.\]
The iterative self-training scheme is initialized $(t=0)$ with a 
base model $\f{0}$ obtained from minimizing the cross-entropy loss on the labeled set which then becomes the first teacher. The
iterative scheme can be described as follows:
\begin{enumerate}[1)]
  \item pseudo-label all unlabeled samples in $\U$ with current teacher $\f{t}$
  \item select a subset $\I \subset \U$ of the pseudo-labeled examples \eg according to their confidence 
  \item train new model $\f{t+1}$ by  minimizing the loss on  the labeled samples in $\T$ and pseudo-labeled samples in $\I$:
  \[ \hspace{-2mm}\frac{1}{n}\sum_{i=1}^n L\big(y_i,\hat{p}_{\f{t+1}}(x_i)\big) + \frac{\lambda}{|\I|}\sum_{z \in \I} L\big(\hat{p}_{\f{t}}(z),\hat{p}_{\f{t+1}}(z)\big)\]
  \item $t\gets t+1$ and go back to step 1
\end{enumerate}
The main difficulty in self-training in a closed world setting is the propagation of labeling mistakes which leads to a degradation of prediction performance. 
In an open world setting an equally severe problem is that a large fraction of the unlabeled instances is not task-relevant such that including them leads to a shift in distribution and can hurt prediction performance (see Figure \ref{fig:teaser} and \ref{fig:cifar10_sample_demonstration}).
The distribution shift is particularly bad regarding AI safety as it yields high confidence predictions on completely unrelated images and can be hard to notice as predictive performance might appear to improve when only evaluated on the test set.

In contrast to previous work which focused on small labeled training sets or problems where the ratio of non-task-related to task-related images is small, our  goal is to show that out-distribution aware training together with a careful sample selection strategy can lead to a self-training scheme which can leverage a large unlabeled dataset to improve performance on the CIFAR10 and CIFAR100 test set over a fully-supervised baseline trained on the entire train set and additionally has excellent OOD detection performance.

\subsection{Out-distribution aware self-training}
A crucial assumption underlying our scheme is that the unlabeled dataset $\U$ contains task-related examples. Moreover, we require an in- and out-distribution validation set to guarantee a high-quality selection and to determine a class-specific criterion to stop the addition of new pseudo-labeled examples. The latter point is often neglected and is particularly important in practice as the number of task-related examples in the unlabeled dataset typically varies significantly between the different classes.

We start with an algorithmic overview over our out-distribution aware self-training (\ours) scheme before describing the individual steps in detail.
\ours is initialized with a 
base teacher model $\f{0}$ trained by minimizing:
\begin{align}\label{eq:base-loss}
\hspace{-2mm}\frac{1}{n}\sum_{i=1}^n L\big(y_i,\hat{p}_{\f{0}}(x_i)\big) + \frac{1}{|\U|}\sum_{z\in\U} L\Big(\frac{1}{K}\ones,\,\hat{p}_{\f{0}}(z)\Big).
\end{align}
We then iterate the following steps starting from $t=0$:
\begin{enumerate}[A)]
  \item calibrate $\f{t}$ on the in-distribution validation set 
  \item pseudo-label all unlabeled samples in $\U$ with current teacher $\f{t}$
  \item for each class $c$: select the top-$k$ unlabeled instances with highest confidence classified as $c$ that lie above the in- and out-distribution thresholds. The selected samples for all classes are denoted as $\I$  
  \item determine new pseudo-labels for the unlabeled instances.
        We use $q(z) = \hat{p}_{\f{t}}(z)$ for $z \in \I$ (selected samples in step C)) and 
  \begin{align}\label{eq:pseudo_labels} v(z) = \frac{1}{2}\Big(\frac{1}{K}+\hat{p}_{\f{t}}(z)\Big), 
  \;  \textrm{ for }\; z \in \U \backslash \I.
  \end{align}
  \item train a new model $\f{t+1}$ by  minimizing the loss on labeled and pseudo-labeled samples:
  \begin{align}\label{eq:final-loss} \frac{1}{n+|\I|} \hspace{-0.5mm}&\Big[\begin{aligned}[t] \sum_{i=1}^n L\big(y_i,\hat{p}_{\f{t+1}}(x_i)\big) 
  \hspace{-0.2mm} + \hspace{-0.5mm}\sum_{z \in \I} \hspace{-0.2mm}L\big(q(z),\hat{p}_{\f{t+1}}(z)\big)\hspace{-0.5mm}\Big] \end{aligned} \nonumber\\  
   + &\frac{1}{|\U\setminus\I|}\sum_{z \in \U \backslash \I} L\big(v(z),\hat{p}_{\f{t+1}}(z)\big)\end{align} 
  \item $t \gets t+1$ and go to step A)
\end{enumerate}


\noindent\textbf{The Base classifier} is essentially an Outlier Exposure (OE) model \cite{HenMazDie2019} (see also \cite{HeiAndBit2019, papadopoulos2019outlier} for related losses) where the set $\U$ can be seen as our training out-distribution where we enforce uniform confidence. OE is known to be one of the best methods for out-of-distribution detection. As in our case a crucial assumption is that the unlabeled samples are partially task-related, it might appear odd to enforce  uniform confidence on all of $\U$. However, we show in Section \ref{sec:th} that this just leads to a down-weighting of the confidence for task-related samples but preserves the Bayes optimal decision and in particular enforces close-to-uniform confidence for all unrelated samples. 


\noindent\textbf{A) Calibration:} while normal neural networks
are known to be overconfident on in-\cite{GuoEtAl2017}
and out-distribution \cite{NguYosClu2015,hendrycks2016baseline,HeiAndBit2019}, the  models resulting from enforcing low confidence on unlabeled points (such as OE) tend to be underconfident on the in-distribution. As we use the predictions of the teacher $\f{t}$ as new soft-labels for the unlabeled data, we calibrate $\f{t}$ by minimizing the expected calibration error using temperature rescaling \cite{GuoEtAl2017}. Thus the teacher model assigns the correct uncertainty score to its predictions on in-distribution samples which improves soft-label quality and  and stabilizes the training procedure.

\noindent\textbf{C) Sample Selection:} The most important problem in self-training is to integrate the right samples into the pseudo-labeled set $\I$. 
While our out-distribution aware teacher is better at discriminating between the in- and out-distribution based on confidence, there are still many samples with highly confident predictions due to the sheer size of the unlabeled dataset ($8\cdot 10^7$). Note that we select at most the top-$k$ samples (where $k=5N(t+1)/K$), but this  might still be too much if not sufficiently many task-related examples of a class exist in the unlabeled dataset. 
We thus need to determine confidence thresholds to limit the selection. 

As the number of available samples per class in the unlabeled dataset is unknown, we calculate a class-specific false-positive based threshold that controls the number of task-irrelevant samples that are falsely added into our pseudo-labeled sample pool $\I$. This is done using an out-distribution validation set, \ie a set of natural images that does not contain any class relevant images (we discuss this choice in Section \ref{sec:exp}). For each class $c$, we compute the $\alpha$-quantile of the predicted probabilities for class $c$ on the out-distribution images which we define as the out-distribution threshold for class $c$ (we use $\alpha=99.8\%$ for CIFAR10).

Similarly, using the in-distribution validation set we define the in-distribution threshold for class $c$  as the smallest predicted probability for class $c$ such that the precision
for all images which are above this threshold is greater than or equal to $\alpha$ (binary classification problem: class $c$ versus all other classes). We use the same $\alpha$ for the in- and out-distribution threshold. An in-distribution precision threshold is especially important if the classification task contains similar classes, as learning with wrong pseudo-labels on task-related images is likely to hurt predictive performance on the in-distribution task even more than the inclusion of an unrelated out-distribution image.

The final per-class threshold is the maximum of the in- and out-distribution threshold. If there are not enough samples among the top-$k$ samples of class $c$ with confidences lying above the threshold, we randomly repeat these samples to maintain a class-balanced training scheme. Note that it is much easier and also more interpretable to fix a precision value rather than the choice of a confidence threshold (in particular if the model is not calibrated) as done in \cite{xie2020self}.

We highlight that that our sample selection strategy can fail or stop the addition of new samples too early if the initial ordering of the unlabeled samples according to the confidence is deficient as we will see for the non-OOD aware self-training in the experiments. This emphasizes the importance of an OOD-aware self-training scheme.

\noindent\textbf{D) Pseudo Labels:} for the original labeled dataset we always use  one-hot labels. For unlabeled data points that have been selected in $\I$, we determine soft-labels $q$ according to the predicted probability distribution over the classes by the calibrated teacher model.
Due to the calibration, this should reflect the ``correct'' uncertainty about these labels. For all remaining images in our unlabeled dataset $\U \backslash \I$ we use a weak form of knowledge distillation by defining soft-labels $v$ as the average of the predicted probability distribution of the teacher model and the uniform distribution, given in  \eqref{eq:pseudo_labels}. This has two reasons: i)
a purely uniform distribution on $\U \backslash \I$, which in the first iterations might still contain a lot of task-relevant images, leads to a bias as it does not distinguish between task-relevant and irrelevant images, ii) only using soft-labels from the teacher model leads to  overconfident predictions as we observe them in the non-out-distribution aware self-training scheme. Thus a trade-off between these opposing goals is their average which leads to heavy damping of the confidence 
(note that the pseudo-labels have a maximal confidence of $\frac{1}{2}+\frac{1}{K}$ on $\U \backslash \I$).

\noindent\textbf{E) Training:} For the final objective in \eqref{eq:final-loss},  the selected pseudo-labeled samples in $\I$ and the original samples in $\T$ are assigned the same weight. This is quite aggressive as we add up to $5$ times more pseudo-labeled data than labeled training data in the first iteration and increase this ratio up to $15$ in the third iteration. However, this also enables larger performance gains given that the sample selection process is successful. Note that the losses on $\I \cup \T$ and on $\U \backslash \I$ have equal weight as the damping of confidences on $\U\backslash\I$ is crucial for the sample selection process.

We iterate this scheme three times. The astonishing part is that while we do not always see monotonic improvements, we never encounter a severe performance degradation. A surprising result, given that we use 80MTI as an unlabeled dataset which was used to create CIFAR10 and CIFAR100
\cite{krizhevsky2009learning} and is known to contain more images of these classes but also many more images not related to CIFAR10 or CIFAR100. This requires us to be highly accurate as there is a large potential to include non-related images.

\subsection{Bayesian Decision Theory of Self-Training}\label{sec:th}
In this section, we analyze our iterations in the framework of Bayesian decision theory. We show that the base classifier that enforces uniform confidence on the unlabeled points still leads to optimal decisions on the in-distribution. Moreover, we show that the iterative scheme with soft-labels ultimately reaches the optimal classifier which is Bayes optimal on the in-distribution task and maximally uncertain elsewhere. Proofs can be found in the Appendix.


We assume that our labeled examples $(x_i,y_i)_{i=1}^n$ are drawn i.i.d. from $\pin(x,y)$. The unlabeled data $(z_i)_{i=1}^m$ is drawn i.i.d. from $\pall(x)$ where we think of $\pall$ in an open world setting as the marginal distribution of a mixture of a very large number of classes (much larger than $K$), including the in-distribution ones.
This also means that $\pin(x)>0$ implies $\pall(x)>0$. 
This assumption on $\pall$ differs from the usual SSL closed world setting where one assumes that the unlabeled examples are also from the $K$ classes or even stronger that they are drawn i.i.d. from $\pin(x)$. 

The \ours base classifier, see \eqref{eq:base-loss}, optimizes
in expectation (for simplicity we omit the index $0$ in $\f{0}$):
\begin{align}\label{eq:exp-oe-loss}
\Exp_{\pin}\Big[L\big(Y,f(X)\big)\Big] + \Exp_{\pall}\Big[L\Big(\frac{1}{K}\ones,f(X)\Big)\Big].
\end{align}
\begin{lemma}
Let $\hat{p}(k|x)=\frac{e^{f_k(x)}}{\sum_{l=1}^K e^{f_l(x)}}$ then the Bayes optimal prediction for the loss \eqref{eq:exp-oe-loss}
is given for any $x$ with $\pall(x)+\pin(x)>0$ as 
\[ \hat{p}(k|x) = \frac{\pin(k|x)\pin(x) + \frac{1}{K}\pall(x)}{\pin(x)+\pall(x)}, \quad k=1,\ldots,K.\]
\end{lemma}
Here we have chosen to directly provide the optimal predictive probability distribution instead of expressing it in terms of the classifier $f$. Note that $\hat{p}(k|x)$ is a monotonic transformation of
$\pin(k|x)$ and thus preserves the ranking of the classes according to $\pin(k|x)$ for each point and does not change the optimal decision. However, the 
absolute ordering of the confidence $\max_k \pin(k|x)$ across different inputs $x$ is influenced significantly by the ratio of $\pin(x)$ to $\pall(x)$. In particular, non-task relevant instances where $\pall(x)$ is larger than $\pin(x)$ are significantly down-weighted and thus will not be selected, whereas if $\pin(x)$ is much larger than $\pall(x)$ the confidence $\max_k \hat{p}(k|x)$ is almost equal to $\max_k \pin(k|x)$. Note that the latter case is in particular true for task-relevant images ($\pin(x)$ large) as $\pall$ is a much more spread out distribution and thus the density value $\pall(x)$ will be small. This justifies our OOD aware initialization and also our post-training calibration step A) as $\hat{p}$ is under-confident on the in-distribution.

\input{figures/cifar_resnet50_acc}

\input{tables/cifar10_results}

The mathematical treatment of our sample selection strategy is difficult, but it is instructive to check the case where at each iteration we impose $t+1$ soft-labels, $\hat{p}_t(k|x)$ defined by the teacher $\f{t}$ at iteration $t$ on all unlabeled points. Then we get the total expected loss
at iteration $t+1$:
\begin{align}\label{eq:exp-iter-loss} \Exp_{\pin}\big[L\big(Y,\f{t+1}(X)\big)\big] + 
\Exp_{\pall}\big[L\big(\hat{p}_{t}(X),\f{t+1}(X)\big)\big].
\end{align}
\begin{lemma}
The Bayes optimal prediction for \eqref{eq:exp-iter-loss} at iteration $t$ for $t\geq 0$
is given for any $x$ with $\pall(x)+\pin(x)>0$ and $k=1,\ldots,K$ as
\begin{align*}
 \hat{p}_{t}(k|x) &=  \pin(k|x)  + 
 \Big(\frac{\pall(x)}{\pin(x)+\pall(x)}\Big)^{t+1}
 \big(\frac{1}{K}-\pin(k|x)\big).
\end{align*}
\end{lemma}
In particular, for any $x$ with $\pin(x)+\pall(x)>0$ we get :
\[ \lim_{t\rightarrow \infty} \hat{p}_t(k|x)=\begin{cases} \pin(k|x) & \textrm{ if } \pin(x)>0\\ \frac{1}{K} & \textrm{ if } \pin(x)=0.\end{cases}\]
Note that this is the perfect out-distribution aware classifier:  Bayes optimal for the in-distribution and maximal uncertainty on all non-task-related regions $(\pin(x)=0)$. 

However, this is just an asymptotic result. In the finite sample case we know that neural networks get overconfident on far away regions \cite{HeiAndBit2019} and thus we need the damping of the soft-labels on the unlabeled part in step D).

\section{Evaluation}\label{sec:exp}

We evaluate our out-distribution aware self-training (\ours) on CIFAR10/100 on two different architectures against two self-training baselines in an open world setting. Moreover, on SVHN we compare \ours in an open world setting against self-training in a closed world setting. More images and ablation studies can be found in the Appendix.

\textbf{Self-training Baselines:} The first baseline \base is a standard self-training scheme, \eg similar to \cite{xie2020self}, not adapted to the open world setting. It follows the steps of \ours with the following differences. In \base we select the unlabeled points in step C) only according to the in-distribution threshold and without any integration of out-distribution knowledge. In the training step E) one just minimizes the cross-entropy loss on the labeled data $\T$ (base classifier) plus pseudo-labeled data $\I$ during the iterations and uses no loss on the remaining unlabeled points $\U \backslash \I$. Apart from the calibration step A) this represents a classical self-training scheme, but \base also profits from more reliable soft-labels.
The second self-training baseline \baseod is partially out-distribution aware in the sense that the selection step C) is the same as in \ours, using both the in- and out-distribution thresholds. \baseod training uses the same loss as \base.

\textbf{Unlabeled dataset:} We use the 80 million tiny images dataset \cite{torralba200880} (denoted as 80MTI) as unlabeled dataset,  which contains $32\times 32$ color images, created by querying 53,464 different nouns from the wordnet hierarchy. Note that CIFAR10 and CIFAR100 are subsets of 80MTI \cite{krizhevsky2009learning} as well as the recent CIFAR10.1 dataset \cite{recht2018cifar10.1} designed as a new test set to assess the generalization of classifiers trained on CIFAR10. We thus remove (near)-duplicates of these datasets from 80MTI,  see the Appendix for details. We note that our removal procedure has higher recall than the one \cite{HenMazDie2019} which we discovered fails to remove all near-duplicates but is less aggressive than the one of \cite{CarEtAl19}, who remove more than $10$ million images just to remove duplicates of 10.000 CIFAR10 test images. After our duplicate removal process, the final 80MTI dataset contains 79106k images (190k images removed). Finally, note that 80MTI has been withdrawn by the authors as it contains a small subset of offensive images \cite{prabhu2020large}. While we respect this decision we decided to continue with this project as our method directly aims to not include information from offensive images that have no connection to the task at hand into its class representation. Additionally, we note that we do not make use of labels that could perpetuate unjust or harmful  stereotypes as we use 80MTI without any form of supervision. 

\textbf{Model Architectures and Training:} we use a standard ResNet50 \cite{he2016deep} and a larger PyramidNet272 \cite{han2017deep} with ShakeDrop regularization \cite{yamada2019shakedrop}. Unless stated otherwise, we train a base teacher model and three student networks and increase the ratio of additionally pseudo-labeled samples to the original training set from 5:1 over 10:1 to 15:1. As in \cite{xie2020self} we inject additional noise into the student training process by using AutoAugment \cite{cubuk18autoaugment} and Cutout \cite{devries17cutout}. ShakeDrop can also be interpreted as noisy student training 
similar to stochastic depth or dropout\cite{SriEtAl2014}. In each iteration, the ResNet50 models are trained from scratch for 250 epochs with a piecewise constant learning rate. 
As the PyramidNet272 requires a substantially larger amount of computational resources, we only train the base model from scratch and use fine-tuning for the student training in later iterations. Additional details 
can be found in the Appendix.

\input{figures/sample_comparison}
\subsection{CIFAR10}
In order to get a validation set without splitting the train set, we use CIFAR10.1 (2k images) as validation set for i) determining the in-distribution threshold for the sample selection step C) and ii) for selecting the best model 
during training. 
As out-distribution for determining the out-distribution threshold in step C) we use CIFAR100 where we removed the classes "bus" and
"pickup-truck" as they can be confused with the classes "car" and "truck" of CIFAR10. For both thresholds we use $\alpha=99.8\%$ which is conservative but justified by the high accuracy of the base CIFAR10 classifiers. Thus it is desirable to add pseudo-labeled data only with very high precision.

In our evaluation, we go beyond the standard test error as this is not the only important property if the final classifier is  applied in an open world setting. Thus we report test error, the error on CIFAR10.1 (note that all methods have  optimized this error as we use this as validation set) and the mean corruption error on CIFAR10-C \cite{hendrycks2019benchmarking} which are 15 different corruptions \eg different kind of noise, blur, contrast and brightness applied in five levels of severity on the CIFAR10 test-set and we report the mean over all corruptions and levels. In an open world setting it is important to be able to distinguish task-related images from the out-distribution. Thus we compute the AUROC values for discriminating based on the confidence between the CIFAR10 test set and the out-distribution datasets: CIFAR100, SVHN, LSUN-CR\cite{LSUN}, Flowers\cite{nilsback2008automated} and Food-101\cite{bossard14food} and report the mean AUROC which we denote as OD-AUROC.

\input{figures/cifar100_samples}
\input{tables/cifar100_accuracies}
\textbf{ResNet50:} In Figure \ref{fig:plots} we report
the test error and number of added pseudo-labeled examples as they evolve over the iterations. All results are summarized in Table \ref{tab:cifar10_results}. \ours is the only method able to improve the test performance over all 3 iterations using  $15$ times the number of original training samples. The two baseline methods \base and \baseod improve the test error until the first resp. second iteration but get worse or stagnate in later iterations. This is due to large errors in the selected samples which lead to a distribution shift as illustrated in Figure \ref{fig:cifar10_sample_demonstration} where we plot images which are \emph{exclusively} selected by \ours and \baseod (\base is even worse). The problem of \baseod is that a lot of out-distribution images are classified with high confidence and thus even the out-distribution threshold cannot prevent that some of them are selected. As the training of \baseod in contrast to our \ours does not enforce low confidence on unlabeled images, the sample selection quality degrades significantly over iterations. In particular, one can notice a distribution shift as almost all images containing humans are classified as "horse" or "dog". Surprisingly, the test performance of \baseod does not degrade more although most selected samples are not class related.
In contrast \ours selects mostly task-related images from 80MTI and stops selecting examples in a class-specific way e.g. \ours selects at most 12k images of the class "frog" as 80MTI does not contain as many "frog" images in comparison to more dominant classes like "cars".  While the catastrophic failure of the sample selection process for \base and \baseod is not apparent from the test error, there is a strong drop in OD-AUROC. In contrast \ours improves the OD-AUROC in the first two iterations which shows that \ours learns a robust representation of the classes. This is also reflected in the CIFAR10.1 error, where \ours achieves a CIFAR10 to CIFAR10.1 gap of $2.82\%$, which is significantly smaller than the smallest gap of 4.1\% reported in \cite{recht2018cifar10.1} and and the improved corruption resistance compared to the base model.

We also trained a ResNet50 with the 500k selected samples from 80MTI of \cite{CarEtAl19} 
and report the results in Table \ref{tab:cifar10_results}. The test error of $3.18\%$ is similar to \base and \baseod in the 2nd iteration (where we also add up to $50k$ per class) but significantly worse than the $2.01\%$ of \ours.

In summary, despite the better base test error, \ours, is able to improve by $1.31\%$ from $3.19\%$ to $1.88\%$ whereas the largest improvement for the two baseline methods is $0.85\%$ ($3.89\%$ to $3.04\%$). Up to our knowledge $1.88\%$ is the best reported performance of a ResNet50 on CIFAR10. Concerning other results with the same augmentation, \cite{cubuk18autoaugment} require a much larger AmoebaNet-B to achieve an error rate of $1.80\%$ and only achieve $2.6\%$ with a WideResNet-28x10, which typically outperforms a ResNet50 \cite{ZagKom2016}. 


\textbf{PyramidNet272:} As the base performance is already below $2\%$ further improvements are much harder to realize and can only be obtained by succeeding in the challenging task to select class relevant samples of very high-quality from the large pool of unlabeled samples. \ours achieves this by improving to $1.31\%$ accuracy which up to our knowledge is the best test accuracy achieved with this architecture (previously $1.36\%$, see \cite{harris2020fmix}). In contrast, due to poor performance in the sample selection \base and \baseod degrade all performance measures from the first iteration on and thus are not able to profit at all from unlabeled data.


\subsection{CIFAR100}
For CIFAR100 we randomly select 50 out of the 500 training samples per class as validation set. As certified out-distribution, 
 we use CIFAR10 without classes "car" and "truck" as they are ambiguous wrt to ``pickup-truck''. In- and out-distribution thresholds are set to $98\%$, due to the lower base accuracy on CIFAR100. AUROC values are calculated wrt CIFAR10, SVHN, LSUN CR and FGVC-Aircraft.   
In Figure \ref{fig:cifar100_samples} and Table \ref{tab:CIFAR100}, one can see that for ResNet50 only \ours is able to improve the test error (by $5.16\%$) throughout iterations and for the PyramidNet \ours is the only one which improves test error at all (by $1.23\%$).
In comparison,  \cite{athiwaratkun2019consistent} report in Table 3 an improvement of $1.17\%$ for 50k add. labels per class on a 13-layer CNN and worse performance for a ResNet26 
(Table 5) even with weak-supervision and $50$ additional labeled points per class. Unlike CIFAR10, the out-distribution threshold employed by \baseod does not lead to improvements over \base. This shows that the out-distribution threshold is only useful if the confidence ranking of the teacher model is reliable. We show a random selection of exclusively added samples by \baseod and \ours (third iteration) in Figure \ref{fig:cifar100_samples}.


\subsection{SVHN}
\input{tables/svhn_resnet_accuracy}
In this experiment, we show that \ours in an open world setting can match the performance of standard closed world self-training . This is done by mixing 521k extra validation samples of SVHN with 80MTI to generate an unlabeled dataset for \ours whereas the closed world self-training scheme only sees the 521k additional SVHN samples as unlabeled data (for more details see the Appendix).
Table \ref{tab:svhn_resnet_error} shows that \ours even outperforms the closed world baseline and comes close to the fully supervised baseline (original training set+521k labeled examples). Table \ref{tab:svhn_resnet_error} also shows that even in the third iteration of \ours only $5.9\%$ of all added samples are not from SVHN. However, that does not mean that these 80TIM samples are wrongly added as they often show digits or digit specific features, see Figure \ref{fig:SVHN-sample} for a random selection of these samples. 




\section{Conclusion} 
 We show that using \ours, it is possible to leverage information from large unlabeled datasets with only a tiny fraction of task-related samples and consistently improve over the supervised baseline on the labeled dataset. The resulting classifiers are more accurate and robust and show better out-distribution detection performance.

{\small
\bibliographystyle{ieee_fullname}
\bibliography{egbib}
}

\clearpage
\appendix

\section{Proofs for Bayesian Decision Theory of
Self-Training}\label{sec:th-proofs}
We provide here the missing proofs of 
Section \ref{sec:th} where we have analyzed our iterations of self-training in the framework of Bayesian decision theory

We repeat the setting so that this section is self-contained. We assume that our labeled examples $(x_i,y_i)_{i=1}^n$ are drawn i.i.d. from $\pin(x,y)$. The unlabeled data $(z_i)_{i=1}^m$ is drawn i.i.d. from $\pall(x)$ where we think of $\pall$ in an open world setting as the marginal distribution of a mixture of a very large number of classes (much larger than $K$), including the in-distribution ones.
This also means that $\pin(x)>0$ implies $\pall(x)>0$. 
This assumption on $\pall$ differs from the usual SSL closed world setting where one assumes that the unlabeled examples are also from the $K$ classes or even stronger that they are drawn i.i.d. from $\pin(x)$. 

The \ours base classifier, see \eqref{eq:base-loss}, optimizes
in expectation (for simplicity we omit the index $0$ in $\f{0}$):
\begin{align}\label{eq:exp-oe-loss-2}
\Exp_{\pin}\Big[L\big(Y,f(X)\big)\Big] + \Exp_{\pall}\Big[L\Big(\frac{1}{K}\ones,f(X)\Big)\Big].
\end{align}
The Bayes optimal prediction has been characterized in the following lemma in 
Section \ref{sec:th}.
\begin{lemma}\label{le:A1}
Let $\hat{p}(k|x)=\frac{e^{f_k(x)}}{\sum_{l=1}^K e^{f_l(x)}}$ then the Bayes optimal prediction for the loss \eqref{eq:exp-oe-loss}
is given for any $x$ with $\pall(x)+\pin(x)>0$ as 
\[ \hat{p}(k|x) = \frac{\pin(k|x)\pin(x) + \frac{1}{K}\pall(x)}{\pin(x)+\pall(x)}, \quad k=1,\ldots,K.\]
\end{lemma}
\begin{proof}
We can write the expected loss in \eqref{eq:exp-oe-loss-2} with the predictive distribution $\hat{p}(k|x)=\frac{e^{f_k(x)}}{\sum_{l=1}^K e^{f_l(x)}}$ and the cross-entropy loss, $L(p,\hat{p})=\sum_k p_k \log(\hat{p}_k)$,  as
\begin{align*}
&\Exp_{\pin}\Big[L\big(Y,\hat{p}(X)\big)\Big] + \Exp_{\pall}\Big[L\Big(\frac{1}{K}\ones,\hat{p}(X)\Big)\Big]\\
=& \int_{\R^d} \pin(x) \sum_{k=1}^K
\pin(k|x) L(e_k,\hat{p}(x)) dx\\& + \int_{\R^d}
\pall(x) \sum_{k=1}^K \frac{1}{K}L(e_k,\hat{p}(x)) dx\\
=& \int_{\R^d} L\Big(\sum_{k=1}^K \big[ \pin(x) \pin(k|x) + \frac{1}{K}\pall(x).\big],\hat{p}(x)\Big)\\
=& \int_{\R^d} -\Big(\sum_{k=1}^K \Big[ \pin(x) \pin(k|x) + \frac{1}{K}\pall(x).\Big]\Big)\\
&\;\Big(f_k(x)-\log(\sum_{l=1}^K e^{f_l(x)})\Big)\\
\end{align*}
where we have used that the cross-entropy loss is linear in the first argument. Moreover, the cross-entropy loss is convex in the second argument and thus the optimality condition
for 
\[ L(p,f)=-\sum_{k=1}^K p_k \Big(f_k-\log\big(\sum_{l=1}^K e^{f_l}\big)\Big).\] is given by
\[ \frac{\partial L}{\partial f_r}=-p_r + \sum_{k=1}^K p_k \frac{e^{f_r}}{\sum_{l=1}^K e^{f_l}}.\]
which yields
\[ \frac{e^{f_r}}{\sum_{l=1}^K e^{f_l}} = \frac{p_r}{\sum_{l=1}^K p_l}.\]
and thus we get
\[ \hat{p}(k|x) = \frac{e^{f_k(x)}}{\sum_{l=1}^K e^{f_l(x)}}= \frac{\pin(x) \pin(k|x) + \frac{1}{K}\pall(x)}{\pin(x) + \pall(x)}.\]
\end{proof}

The second result yields the Bayes optimal prediction for an interated training scheme where the predictions of the teacher $\f{t}$ at iteration $t$ become the soft-lables for the student model $\f{t+1}$
Then we get the total expected loss
for the student model $\f{t+1}$ at iteration $t+1$:
\begin{align}\label{eq:exp-iter-loss} \Exp_{\pin}\big[L\big(Y,\f{t+1}(X)\big)\big] + 
\Exp_{\pall}\big[L\big(\hat{p}_{t}(X),\f{t+1}(X)\big)\big].
\end{align}
\begin{lemma}
The Bayes optimal prediction for \eqref{eq:exp-iter-loss} at iteration $t$ for $t\geq 0$
is given for any $x$ with $\pall(x)+\pin(x)>0$ and $k=1,\ldots,K$ as
\begin{align*}
 \hat{p}_{t}(k|x) &=  \pin(k|x)  + 
 \Big(\frac{\pall(x)}{\pin(x)+\pall(x)}\Big)^{t+1}
 \big(\frac{1}{K}-\pin(k|x)\big)
\end{align*}
\end{lemma}
\begin{proof}
We prove this by induction. First we note that for $t=0$ (base model) we have shown the predictive distribution $\hat{p}_0(k|x)$ in Lemma \ref{le:A1} to be
\[ \hat{p}_0(k|x) = \frac{e^{f_k(x)}}{\sum_{l=1}^K e^{f_l(x)}}= \frac{\pin(x) \pin(k|x) + \frac{1}{K}\pall(x)}{\pin(x) + \pall(x)}.\]
Noting that 
\begin{align*}
&\frac{\pin(x) \pin(k|x) + \frac{1}{K}\pall(x)}{\pin(x) + \pall(x)}\\ =& \pin(k|x) +\Big(\frac{\pall(x)}{\pin(x)+\pall(x)}\Big)^{1}
 \big(\frac{1}{K}-\pin(k|x)\big).
 \end{align*}
we observe that this agree with the expression in the lemma to prove. Thus the induction start for $t=0$ is verified and we go for the induction step. Repeating the derivation of 
Lemma \ref{le:A1} we get
\begin{align*}
\hat{p}_{t+1}(k|x) = \frac{\pin(x)\pin(k|x) + \pall(x)\hat{p}_t(k|x)}{\pin(x)+ \pall(x)}
\end{align*}
From plugging in the induction hypothesis we get
\begin{align*}
& \frac{\pin(x)\pin(k|x) + \pall(x)\hat{p}_t(k|x)}{\pin(x)+ \pall(x)}\\
=& \frac{\pin(x)\pin(k|x)}{\pin(x)+ \pall(x)} + \frac{ \pall(x)}{\pin(x)+ \pall(x)} \Big[  \pin(k|x) \\
& \ + \Big(\frac{\pall(x)}{\pin(x)+\pall(x)}\Big)^{t+1}
 \big(\frac{1}{K}-\pin(k|x)\big) \ \Big]\\
=& \frac{\pin(x)\pin(k|x) + \pall(x)\pin(k|x)}{\pin(x)+ \pall(x)} \\ 
& \ + \Big(\frac{\pall(x)}{\pin(x)+\pall(x)}\Big)^{t+2}
 \big(\frac{1}{K}-\pin(k|x)\big) \\
=& \pin(k|x) + \Big(\frac{\pall(x)}{\pin(x)+\pall(x)}\Big)^{t+2}
 \big(\frac{1}{K}-\pin(k|x)\big) 
\end{align*}
which finishes the proof.
\end{proof}
In particular, for any $x$ with $\pin(x)+\pall(x)>0$ we get :
\[ \lim_{t\rightarrow \infty} \hat{p}_t(k|x)=\begin{cases} \pin(k|x) & \textrm{ if } \pin(x)>0\\ \frac{1}{K} & \textrm{ if } \pin(x)=0.\end{cases}\]
Note that this is the perfect out-distribution aware classifier:  Bayes optimal for the in-distribution and maximal uncertainty on all non-task-related regions $(\pin(x)=0)$. 

\section{Duplicate removal}
In this section, we explain our approach to duplicate removal of CIFAR test images in the 80MTI dataset. As both CIFAR10 and CIFAR100 are subsets of 80 million tiny images, it is important to remove exact- and near-duplicates of test images from the unlabeled distribution to prevent them from leaking into our train set. First, we noticed that the duplicate removal from \cite{HenMazDie2019} did not remove all duplicates from 80MTI, which they use as out-distribution to enforce uniform confidence (Figure \ref{fig:appendix_duplicates_hendrycks}). While this should not improve their test accuracy, it might influence out-distribution detection when for example calculating the AUROC between the CIFAR10 and CIFAR100 test sets. \cite{CarEtAl19} follow \cite{recht2018cifar10.1} and remove all 80MTI images with an $l_2$-distance smaller than $2000/255$ to the nearest neighbour in the CIFAR10 test set. While this is likely to remove all duplicates, the approach seems overly strict as after this process, only 65.807.640 out of the 79.302.017 images remain. Thus with their definition of near-duplicate, 80MTI contains nearly 14 million duplicates of the 10.000 test images. In Figure \ref{fig:appendix_duplicates_cifar_784}, we show that 
almost all excluded images are no true duplicates. While we acknowledge that it is important to optimise recall instead of precision when removing duplicates, their procedure is too aggressive and leads to an exclusion of a large set of images which have low variation or close to monochrome images. 
In Figure \ref{fig:appendix_duplicate_hist}we show a histogram of the $l_2$-nearest neighbor distances between the CIFAR test sets and 80MTI. One can see that the vast majority of images 
have a nearest neighbor distance above an $l_2$-distance of $3.0$. A visual inspection in Figure \ref{fig:appendix_duplicates_cifar_3} also confirms that most images below that threshold are duplicates, thus we first remove all images from 80MTI with an $l_2$- distance less than $3.0$ to one of the CIFAR test images. While this removes all exact duplicates, there can exist near-duplicates with larger $l_2$-distance. We thus collect all 80MTI samples with a $l_2$-nearest neighbour to the CIFAR test set smaller than $2000/255$ as potential candidates for removal. For each candidate $x$ and nearest neigbhour $z$ in the CIFAR test sets, we then calculate the perceptual similarity metric LPIPS \cite{zhang2018unreasonable} and SSIM \cite{wang2004image} and remove the image if $\text{LPIPS}(x,z) < 0.025$ and $1 - \text{SSIM}(x,z) < 0.4$. As both metrics are closer to the visual system, we found them to be more reliable at finding near duplicates for images with larger $l_2$ distances, see Figure \ref{fig:appendix_duplicates_cifar10_784}, but they are much too expensive to use them directly for nearest neighbor search. We highlight that we do not only find exact duplicates but also degraded versions that for example contain blur, slight translations, color changes and added text or logos. Note that we still remove some non-duplicates, showing that our thresholds are still chosen rather conservatively. Overall, we remove 24k CIFAR10 test set duplicates and 60k CIFAR100 test set duplicates. Additionally, we also remove all samples selected by \cite{HenMazDie2019}, which in particular includes exact train set duplicates. When training CIFAR10 models, we also remove all CIFAR10.1 duplicates with the same approach. 

\input{figures/app_duplicate_hist}

\input{figures/app_duplicate_removal}

\section{Sample comparison}
\subsection{CIFAR10}
In this section, we present a random selection of class specific samples selected for the set $\I$ (see sample selection step C)) for the three iterations of \ours, \base and \baseod. Note that unlike in Figure \ref{fig:cifar10_sample_demonstration} in the main paper, we show randomly selected samples from $\I$ selected by the respective method and do \emph{not} restrict our selection to exclusively selected samples by \ours versus \baseod. In Figure \ref{fig:appendix_cifar10_odst}, we show results for ResNet50 \ours, \base baseline in Figure \ref{fig:appendix_cifar10_st} and \baseod with OD-thresholding in \ref{fig:appendix_cifar10_st_ot}.  Additionally, we provide the number of samples above our computed thresholds for \ours, \base and \baseod in Table \ref{tab:app_num_samples_cifar10}. Note that we use maximally up to 25k, 50k or 75k per class in the first, second and third iteration respectively but there can be more samples which have confidence higher than our class-specific thresholds.  Our in- and out-distribution thresholds used in \ours   result in a rather conservative selection of samples, selecting less than 20k samples for both "frog" and "deer", less than 75k for "cat", "dog" and "truck" and the full 75k for the remaining classes in the third iteration. In particular, for difficult classes like frog and deer we found our selected samples above the threshold to be of very high quality while the model makes more and more false-positive predictions below the threshold. Thus, the combination of our in-and out-distribution thresholds plus the out-distribution aware training leads to a very good sample selection quality with a small number of false positives. 

The \base baseline without OD-thresholding on the other hand accepts way too many samples above the ID-threshold which results in wrong class representations in the later iterations. For example, the model associates flags with "plane", human faces with "dog" and human portraits with "horse". Even with the OD-threshold, \baseod seems to systematically make similar predictions, highlighting that both thresholding and OD-aware training are necessary for a successful sample selection.
We additionally compare samples from 500k-TI \cite{CarEtAl19} (which select roughly 50k per class) to our selection of maximal 50k samples per class in Figure \ref{fig:appendix_cifar10_carmon}.
\cite{CarEtAl19} use a $K+1$ class model trained on a labeled subset of 80 million tiny images which does not contain images related to CIFAR10 to remove out-of-distribution samples. Even with this form of weak-supervision, they achieve worse quality for underrepresented classes such as "deer" and "frog" as they include 50k samples for each class. While their sample selection is good for the most part, they make some occasional mistakes, for example they include a street sign in front of a clouded sky as "plane" and a reptile as "bird". We note that both \ours and 500k-TI include some related classes like trains and busses into the category "truck".

\subsection{CIFAR100}
For CIFAR100 we show four samples per class for each of the 100 classes for each iteration with the ResNet architecture. Figures \ref{fig:appendix_cifar100_odst_1}
to \ref{fig:appendix_cifar100_odst_3} contain the results for \ours, and Figures \ref{fig:appendix_cifar100_st_1} to \ref{fig:appendix_cifar100_st_ot_3} for \base and \baseod. The number of added samples per class is visualized in Figure \ref{fig:appendix_cifar100_samples_hist}. We again highlight the difficulty of the task at hand. The model has to select additional samples for 100 classes with each only having 450 train samples from a pool of 80 million images that are mostly not task-related.  Despite this fact, \ours is able to select images for most classes with high accuracy throughout all 3 iterations. We highlight the large amount of diversity in the selected samples for most classes, which is important for proper generalization performance. While the average sample quality across classes is very good, few classes such as "worm" are problematic as the model selects visually similar but non-relevant samples. On the other hand, \ours is able to distinguish similar classes such as "leopard" and "lion", "dinosaur" and "elephant" or "apple" and "orange" and select proper samples for each of them. 

Both baseline methods \base and \baseod are performing surprisingly well in the first iteration. We believe that this is mostly due to the relatively small number of samples that are added in the first iteration (2250 per class). However, without OD-aware training, the quality of selected samples decreases progressively for both baselines in the second and third iteration. Once again, one can see that the model learns completely wrong class representations, for example \base starts to associate cooked food with "crab", probably as some "crab" images in the train set contain plates. Such errors accumulate, as the teacher passes on those wrong representations to its student and there is no correction mechanism that could prevent the student from learning those wrong representations. It is thus extremely important to be conservative with the addition of new samples, although OD-thresholding is not sufficient on its own, as the \baseod baseline still suffers from similar problems as \base. We again emphasise that the model learning wrong class representations can not necessarily be observed on the test set and it is thus important to judge open world SSL algorithms not only based on their test predictive performance.   

\input{tables/app_num_samples}
\input{figures/app_cifar100_samples_histogram}

\section{Ablation}
In this section, we present various ablation studies to motivate our design choices. All ablation studies use the ResNet50 architecture and are done on CIFAR10 and CIFAR100. 

\input{tables/ablation_cifar_non_iterative}

\input{tables/ablation_cifar_cominbed}

\subsection{Non-iterative training}
As iterative self-training can greatly increase computational cost, it is obvious to ask whether one could skip the first two iterations and directly train with up to 15 times the amount of pseudo-labeled data per class. We compare the third iteration of \ours model with a non-iterative model that is directly trained with up to 15 times the amount of pseudo-labeled data in Table \ref{tab:ablation_cifar10_non_iterative}. While for CIFAR10 non-iterative training performs similarly well to iterative training, with slightly worse test but better robust accuracy, non-iterative training greatly decreases performance on CIFAR100. This highlights that especially for more complex tasks, iterative training is necessary to achieve the best performance. 

\subsection{Choice of peudo-labels on $\U \setminus \I$}
While \cite{xie2020self} demonstrated that soft-labels on selected samples can improve performance, it remains an open question whether one should use soft-labels for the remaining samples in $\U \setminus \I$. In principle, one could keep enforcing uniform confidence on $\U \setminus \I$, that is
\[ v(z)_i=\frac{1}{K}, \; i=1,\ldots,K \textrm{ for } z \in \U \backslash \I,\] like we did for training the baseline model. Thus instead of training the students with the loss presented in \eqref{eq:final-loss},we minimize the loss:
\begin{align}\frac{1}{n+|\I|} \hspace{-0.5mm}&\Big[\begin{aligned}[t] \sum_{i=1}^n L\big(y_i,\hat{p}_{\f{t+1}}(x_i)\big) 
  \hspace{-0.2mm} + \hspace{-0.5mm}\sum_{z \in \I} \hspace{-0.2mm}L\big(q(z),\hat{p}_{\f{t+1}}(z)\big)\hspace{-0.5mm}\Big] \end{aligned} \nonumber\\  
   + &\frac{1}{|\U\setminus\I|}\sum_{z \in \U \backslash \I} L\big(\mathbf{1}/K,\hat{p}_{\f{t+1}}(z)\big).\end{align} 
   
The other alternative is that we use soft-labels on $\U \backslash \I$:
\begin{align} 
v(z) = \hat{p}_{\f{t}}(z) \textrm{ for }\; z \in \U \backslash \I.
\end{align}
In Table \ref{tab:ablation_cifar_combined} we compare these alternative choices to the one of \ours (the mean of both) given in Equation \eqref{eq:pseudo_labels} which shows that our chosen pseudo-labels in \ours are the right compromise between these two extremes.

Using hard-labels on $\U \setminus \I$ decreases performance in comparison to \ours, especially on CIFAR100. There are two possible explanations for this. First, due to our strict thresholding, it is possible that $\U \setminus \I$ contains task-relevant examples that are correctly classified but not accepted into $\I$. In this case, the soft-label is a better target for the next student than strict uniform confidence. Second, even for unrelated images that contain certain features that correlate with a specific class, soft-labels might be a more meaningful target.
Soft-labels without label smoothing clearly outperform hard labels, however the OD-AUROC values reveal that the model becomes increasingly overconfident on out-distribution samples, especially for CIFAR10. This results in a decrease in sample selection quality which again causes worse overall performance than \ours. 

\section{SVHN Experiments}
Next we give a more detailed overview over our SVHN experiments. For SVHN on top of the 73257 standard training samples there are 531k additional labeled samples available. We split them into 10k validation and 521k unlabeled samples. The unlabeled set $\U$ for \ours is the union of the entire 80MTI dataset with the 521k unlabeled samples. Disregarding additional numbers in 80MTI, this results in a task-related ratio of $0.65\%$.\\
To compare \ours open world self-training with standard close world self-training, we train an additional ST-CW model. The base model is trained by minimizing the cross-entropy on the 73k labeled samples and the self-training iterations only select new samples from the extra 521k unlabeled samples. For both \ours and ST-CW, we select up to 25k samples per class in the first iteration and up to 50k and 75k in the second resp. third iteration. \ours uses both the OD- and ID-threshold and ST-CW only uses the precision ID-threshold. Moreover, we also train a fully supervised model on the 73k train samples plus 521k extra samples with labels.  
For all SVHN experiments, we use AutoAugment with Cutout. 

Even when compared to the close world setting, \ours is able to not only match but even outperform the self-training baseline (Table \ref{tab:svhn_resnet_error}). We also note that the self-training baseline accepts 504k unlabeled samples in a close world setting and \ours recalls 502k out of the 521k SVHN extra in a pool of 80 million images.

\section{Implementation details}
In this section we present the hyperparameters used to train our models. Note that we use the exact same set of hyperparameters for \ours and the two baselines \base and \baseod. 

\subsection{ResNet50}
Our ResNet50 models are trained for 250 epochs with piecewise learning rate schedule. We use a batchsize of 128, a starting learning rate of 0.1 and decay it by a factor of 10 at epochs 100, 150 and 200. We use Nesterov stochastic gradient descent optimizer with a momentum weight of $0.9$. The weight decay is set to 0.0005. Throughout all iterations we use AutoAugment \cite{cubuk18autoaugment} and Cutout \cite{devries17cutout}. We evaluate validation set error throughout the last 20\% of epochs and chose the model with the best validation set performance.

\subsection{Shakedrop PyramidNet272}
The base PyramidNet is trained for 1000 epochs with a cosine schedule and initial learning rate of 0.05 and a batch size of 64. We use Nesterov SGD with a weight decay of 0.0001. The base model is trained with with AutoAugment and Cutout.  

As training large models for 1000 epochs and up to 16 times the original amount of data is expensive, we use fine-tuning to train the later student models for 55 epochs. We thus always initialize the new student with the previous teacher model's weight. Note that although the student model is initialized with the teacher model that was used to label the unlabeled data, the training loss for the pseudo-labeled data is not 0 due to Shakedrop \cite{yamada2019shakedrop} and data augmentation. Due to large amounts of noise from both heavy data augmentation and Shakedrop, we use a mixed augmentation strategy for fine-tuning. In detail we train each student for $N$ epochs with AutoAugment and a cosine schedule with initial learning rate of 0.05 that decays to 0 after $N$ epochs. We then do a warm restart and train for another $55 - N$ epochs using a cosine schedule starting at learning rate 0.01 and only use random cropping and flipping. As data augmentation becomes less useful with increasing amounts of data, we set $N$ to 50 for the first student and decrease it to 40 and 30 for the second and last student. Batch size and weight decay remain at 64 and 0.0001 for fine-tuning.

\input{figures/app_cifa10_samples}
\input{figures/app_cifar100_samples}

\end{document}

%% file: figures/intro_opener.tex



\begin{figure}[t]
\centering
    \newcolumntype{A}{>{\centering\arraybackslash}m{4pt}}
    \newcolumntype{B}{>{\centering\arraybackslash}m{0.88\columnwidth}}
    \begin{tabular}{|A||B|}
        \hline
        \hspace{-5pt}
        \rotatebox[]{90}{ST} & 
        \vspace{-9pt}
        \makecell{\vspace{-9pt}\hspace{-12pt}\includegraphics[width=0.98\columnwidth]{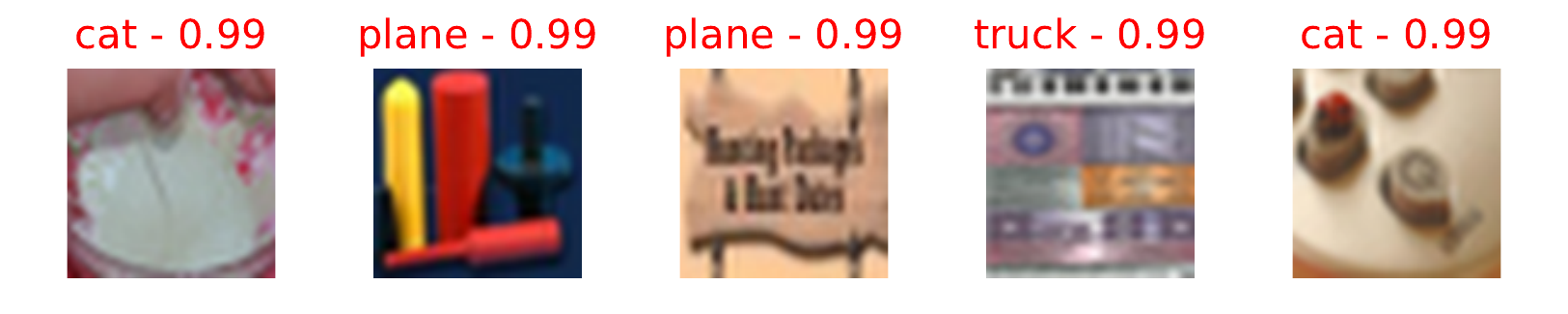}}\\
        \hline
        \hspace{-5pt}
        \rotatebox[]{90}{\hspace{10pt}\ours} &
        \vspace{-9pt}
        \makecell{\vspace{-9pt}\hspace{-12pt}\includegraphics[width=0.98\columnwidth]{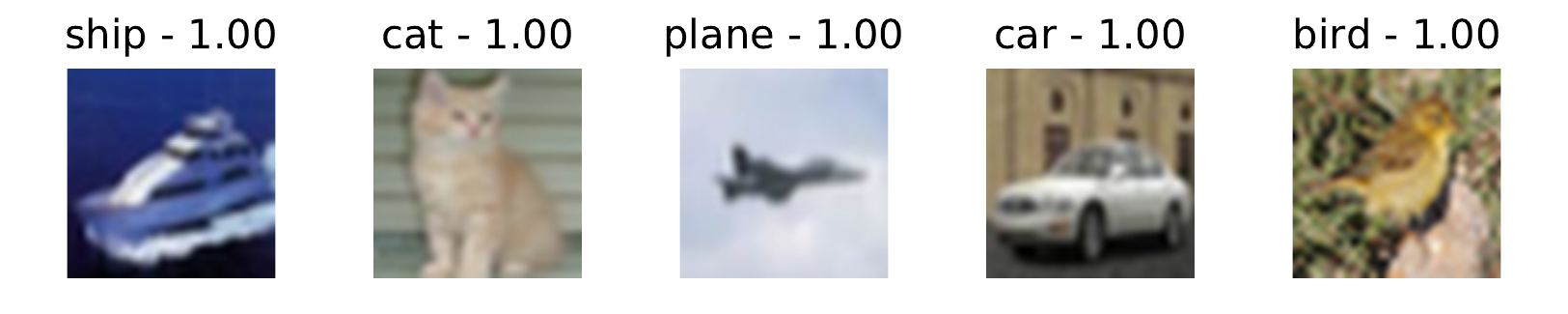}}\\
        \hline
    \end{tabular}

\caption{\label{fig:teaser}
Standard self-training (ST) fails to select the right unlabeled samples in an open world setting but our out-distribution aware self-training (\ours) (bottom) has very good selection quality (see also Figure \ref{fig:cifar10_sample_demonstration}).
}

\end{figure}

%% file: figures/cifar_resnet50_acc.tex
\begin{figure*}
\centering
\begin{subfigure}{0.245\textwidth}
    \centering
    \includegraphics[width=\textwidth]{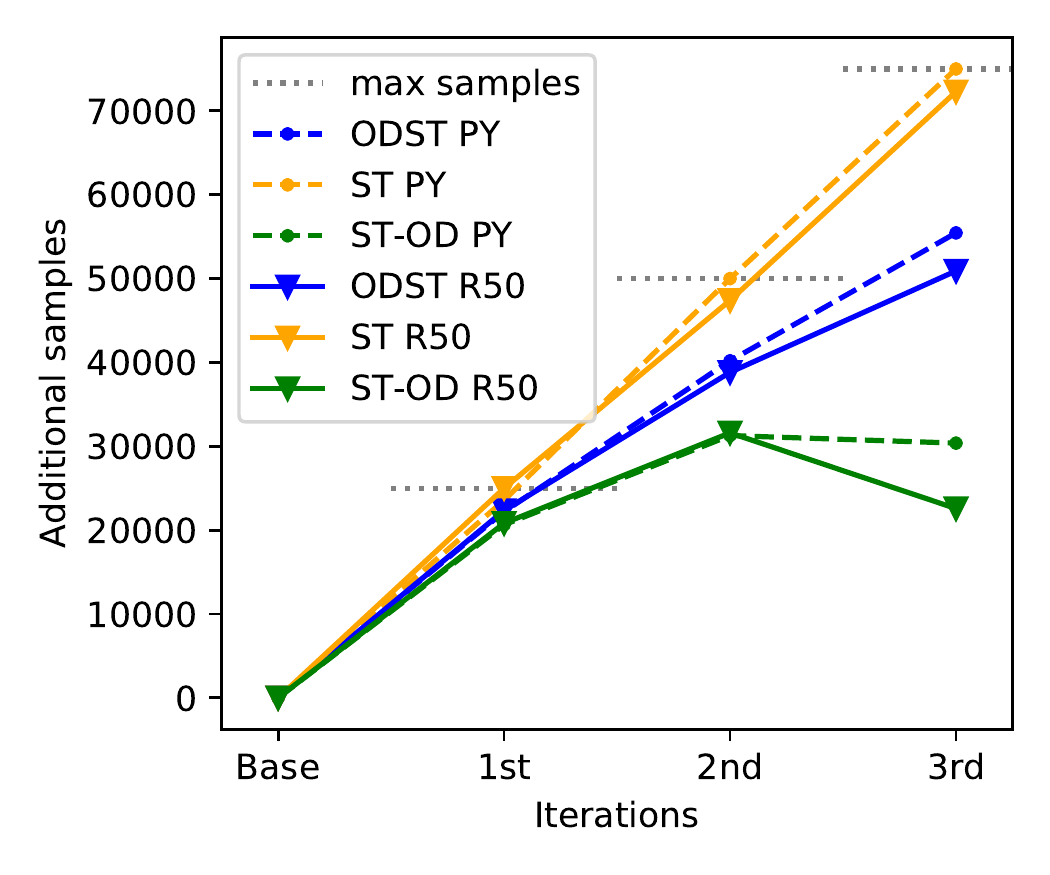}
    \caption{Added Samples for CIFAR10}
    \label{subfig:cifar10_test_error_iterations}
\end{subfigure}
    \begin{subfigure}{0.245\textwidth}
    \centering
    \includegraphics[width=\textwidth]{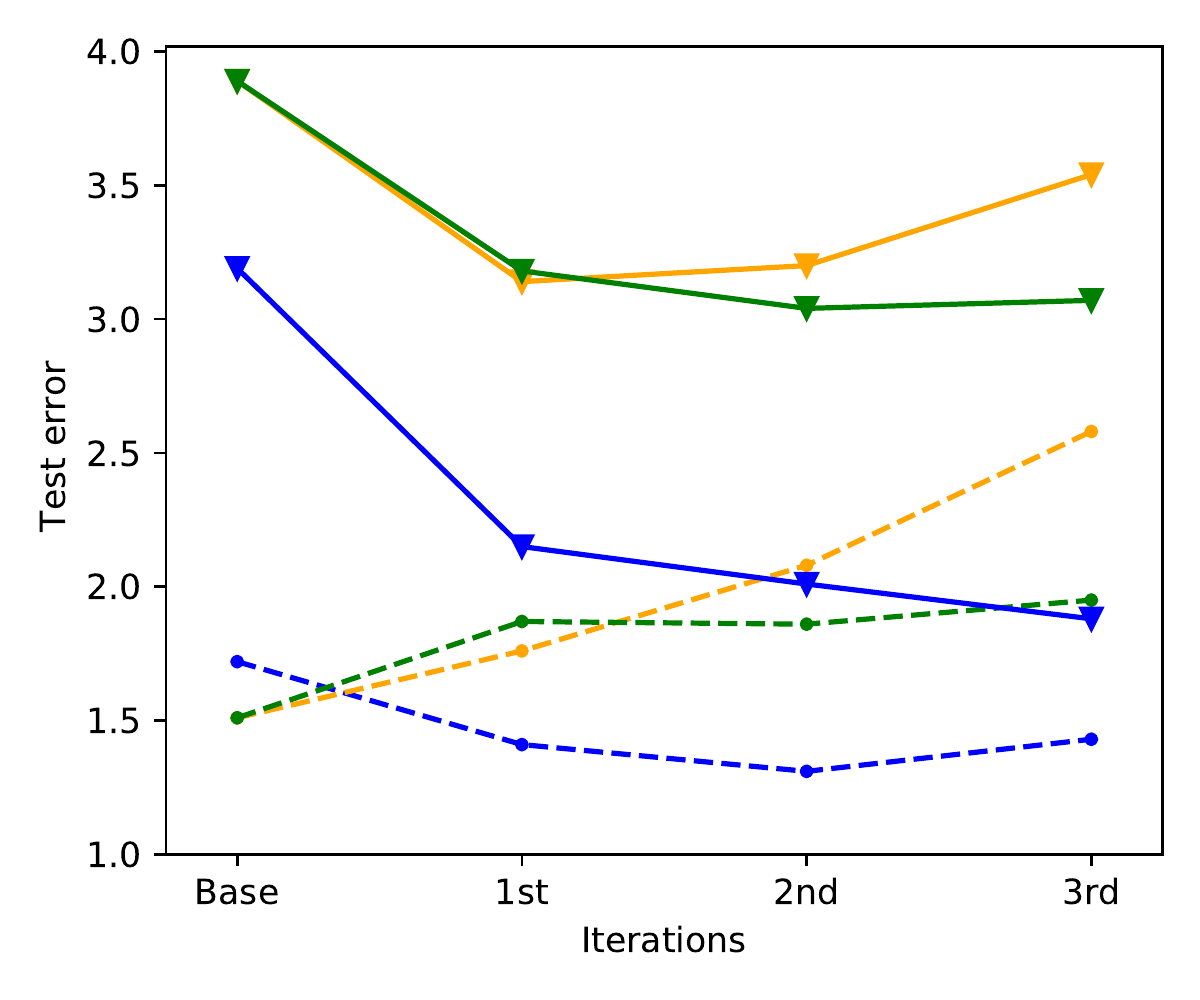}
    \caption{Test Error for CIFAR10}
    \label{subfig:cifar10_samples_iterations}
\end{subfigure}
\begin{subfigure}{0.245\textwidth}
    \centering
    \includegraphics[width=\textwidth]{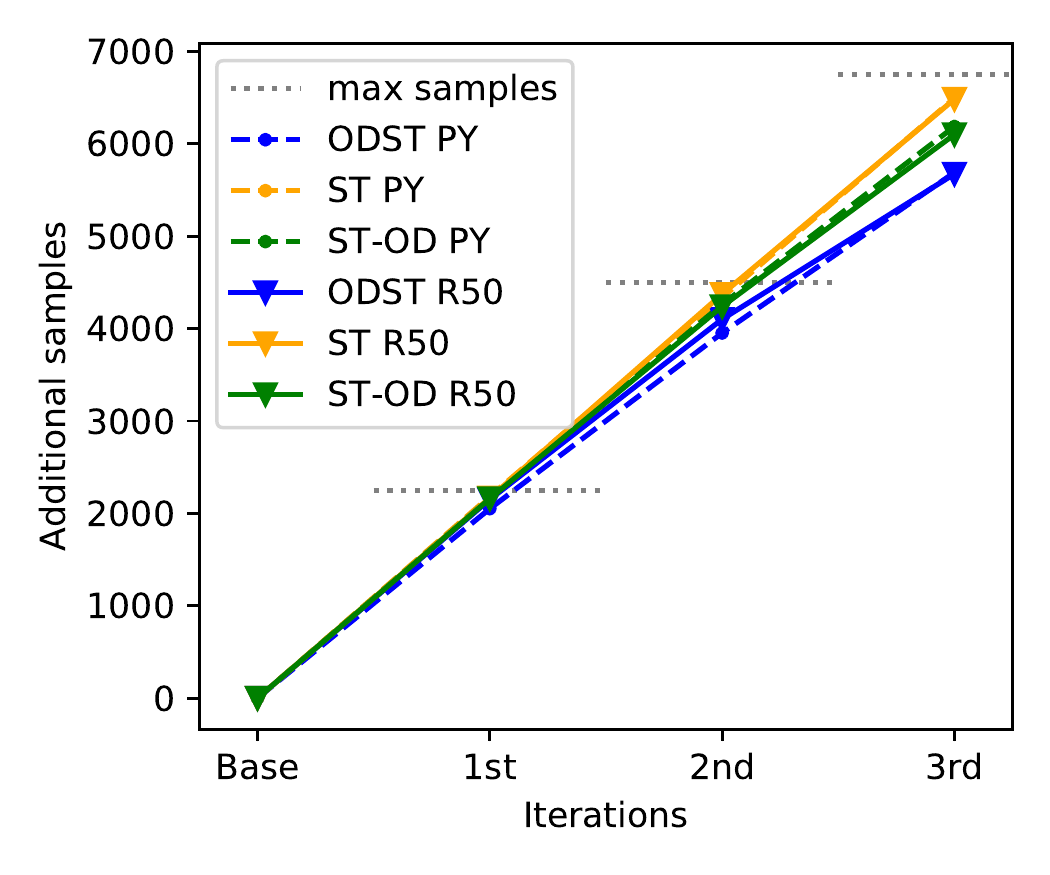}
    \caption{Added Samples for CIFAR100}
    \label{subfig:cifar100_test_error_iterations}
\end{subfigure}
\begin{subfigure}{0.245\textwidth}
    \centering
    \includegraphics[width=\textwidth]{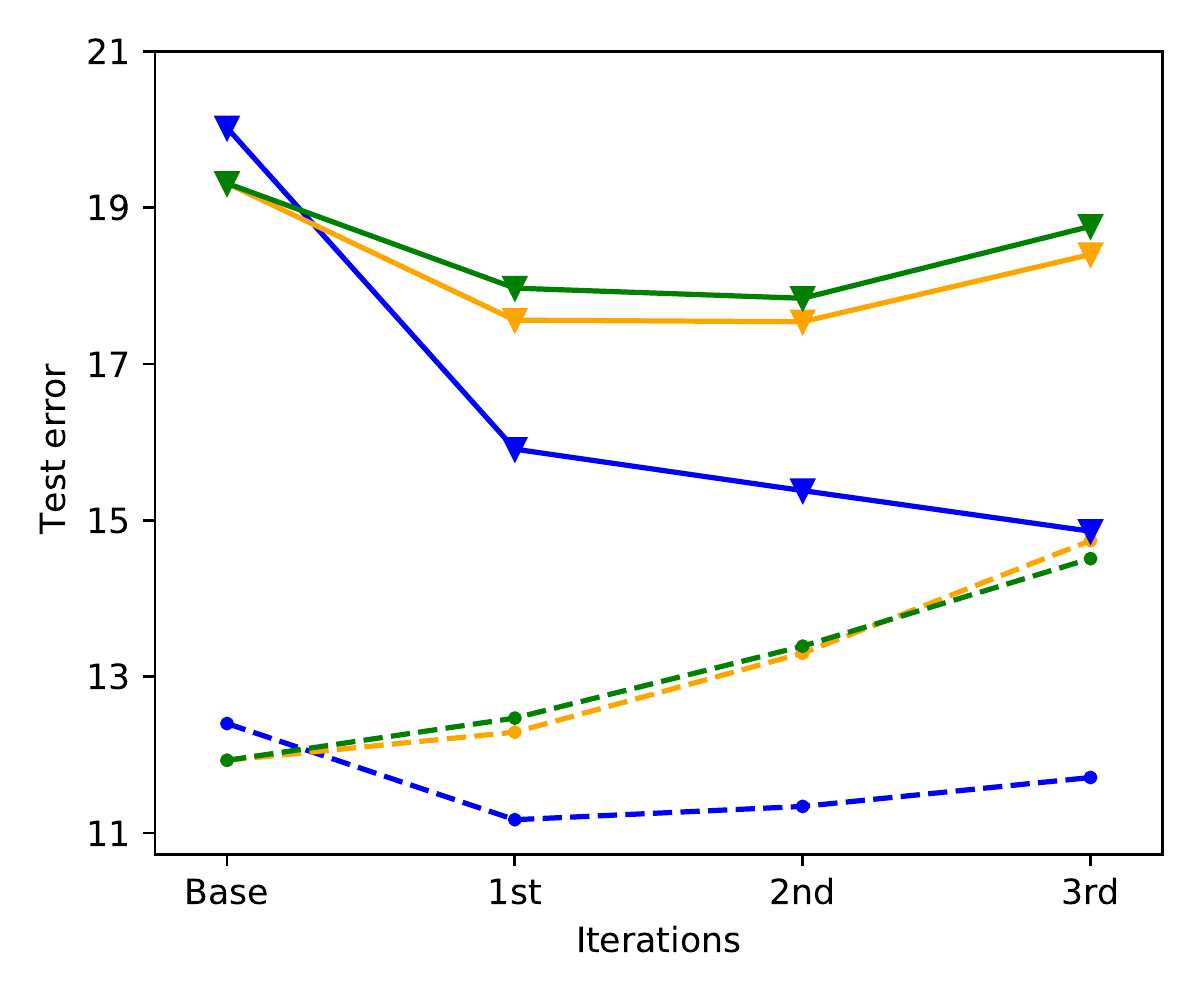}
    \caption{Test Error for CIFAR100}
    \label{subfig:cifar10_test_error_iterations}
\end{subfigure}
\caption{\label{fig:plots}Mean of added unlabeled  samples per class and test error over iterations for CIFAR10 (left) and CIFAR100 (right).}

\end{figure*}

%% file: tables/cifar10_results.tex
\begin{table*}[]
    \newcolumntype{M}{>{\centering\arraybackslash}m{0.55cm}}
    \small
    \centering
    \begin{tabular}{|c|*{4}{M}|*{4}{M}|*{4}{M}|*{4}{M}|}
    \hhline{~|----|----|----|----|}
    \multicolumn{1}{c|}{} & \multicolumn{4}{c|}{CIFAR10 error} & \multicolumn{4}{c|}{CIFAR10.1 error} & \multicolumn{4}{c|}{CIFAR10-C error } & \multicolumn{4}{c|}{OD-AUROC }\\
    \hline
    \textbf{ResNet} & Base & 1st & 2nd & 3rd & Base & 1st & 2nd & 3rd & Base & 1st & 2nd & 3rd & Base & 1st & 2nd & 3rd\\
    \hhline{|=|====|====|====|====|}
    \ours & \emph{3.19} & \emph{2.15} & \emph{2.01} & \textbf{1.88} & \emph{7.00} & \emph{5.40} & \textbf{4.55} & \emph{4.70} & \emph{16.86} & \textbf{14.19} & \emph{15.02} & \emph{15.46} & \emph{98.95} & \textbf{99.20} & \emph{99.16} & \emph{98.98}\\
    \hline
    \base & 3.89 & 3.14 & 3.20 & 3.54 & 8.65 & 6.80 & 8.00 & 7.75 & 16.97 & 16.41 & 17.36 & 19.04 & 92.74 & 89.39 & 83.73 & 75.93\\
    \hline
    \baseod & 3.89 & 3.18 & 3.04 & 3.07 & 8.65 & 7.25 & 7.30 & 6.85 & 16.97 & 16.53 & 17.62 & 18.98 & 92.74 & 90.35 & 84.14 & 85.71\\
    \hline
    500k-TI \cite{CarEtAl19} & \multicolumn{4}{c|}{ 3.18 (50k Samples/class)} & \multicolumn{4}{c|}{ 7.25}  & \multicolumn{4}{c|}{ 17.35}  & \multicolumn{4}{c|}{ 94.34}\\
    \hhline{|=|====|====|====|====|}
    \textbf{Pyramid} & Base & 1st & 2nd & 3rd & Base & 1st & 2nd & 3rd & Base & 1st & 2nd & 3rd & Base & 1st & 2nd & 3rd\\
    \hhline{|=|====|====|====|====|}
    \ours & 1.72 & \emph{1.41} & \textbf{1.31} & \emph{1.43} & 4.40 & \emph{4.10} & \emph{3.50} & \textbf{3.00} & 13.93 & \emph{13.05} & \emph{12.88} & \emph{13.85} & \emph{99.30} & \emph{99.43} & \textbf{99.44} & \emph{99.38}\\
    \hline
    \base & \emph{1.51} & 1.76 & 2.08 & 2.58 & \emph{3.70} & 5.05 & 5.30 & 6.80 & \textbf{12.21} & 14.72 & 16.12 & 21.12 & 95.43 & 92.49 & 90.07 & 87.10 \\
    \hline
    \baseod & \emph{1.51} & 1.87 & 1.86 & 1.95 & \emph{3.70} & 4.50 & 4.80 & 4.65 & \textbf{12.21} & 14.03 & 15.75 & 18.80 & 95.43 & 92.51 & 91.51 & 91.55\\
    \hline
    \end{tabular}
    \caption{CIFAR10 - test error on CIFAR10 and CIFAR10.1, mean corruption error on CIFAR10-C and out-of-distribution detection performance OD-AUROC. \ours has the best improvement ($1.31\%)$ and final test error ($1.88\%$) for Resnet50 and is the only self-training method which improves for the Pyramid272 architecture by $0.41\%$ with $1.31\%$ test error. }
    \label{tab:cifar10_results}
\end{table*}

%% file: figures/sample_comparison.tex
\def\sc_img_width{0.313\textwidth}
\def\line_separator{0.5pt}
\begin{figure*}
\small
\setlength\tabcolsep{2.0pt} 
\begin{tabular}{ |  c | c | c | c |}
\cline{2-4}
\multicolumn{1}{c|}{}& \multicolumn{1}{c|}{1st Iteration (25k)} & \multicolumn{1}{c|}{2nd Iteration (50k)} & \multicolumn{1}{c|}{3rd Iteration (75k)}\\
\hline
\rotatebox{90}{\hspace{-1.1cm}Baseline (\baseod)} & 
\makecell{\includegraphics[width=\sc_img_width]{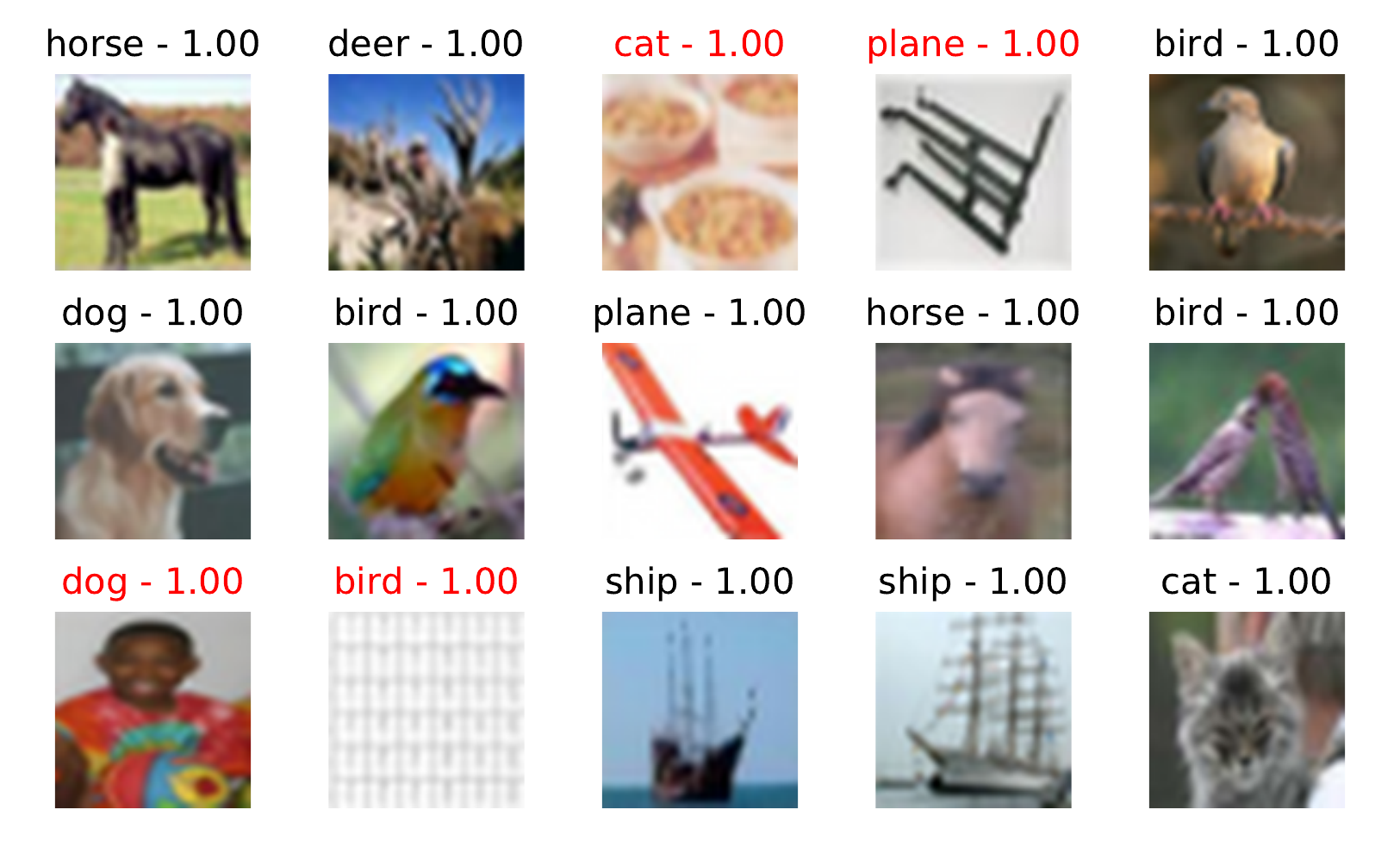}
} &
\makecell{\includegraphics[width=\sc_img_width]{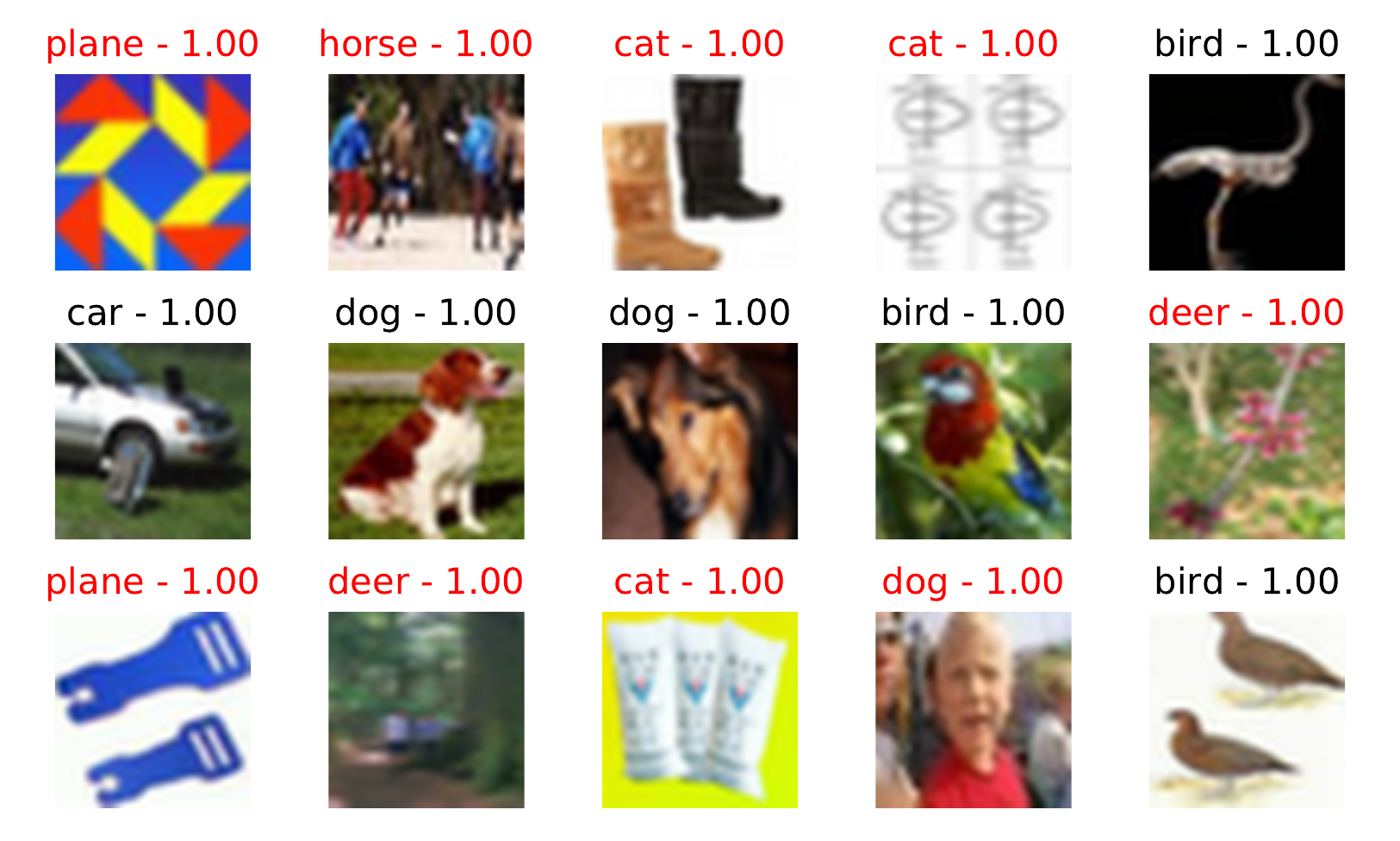}
} &
\makecell{\includegraphics[width=\sc_img_width]{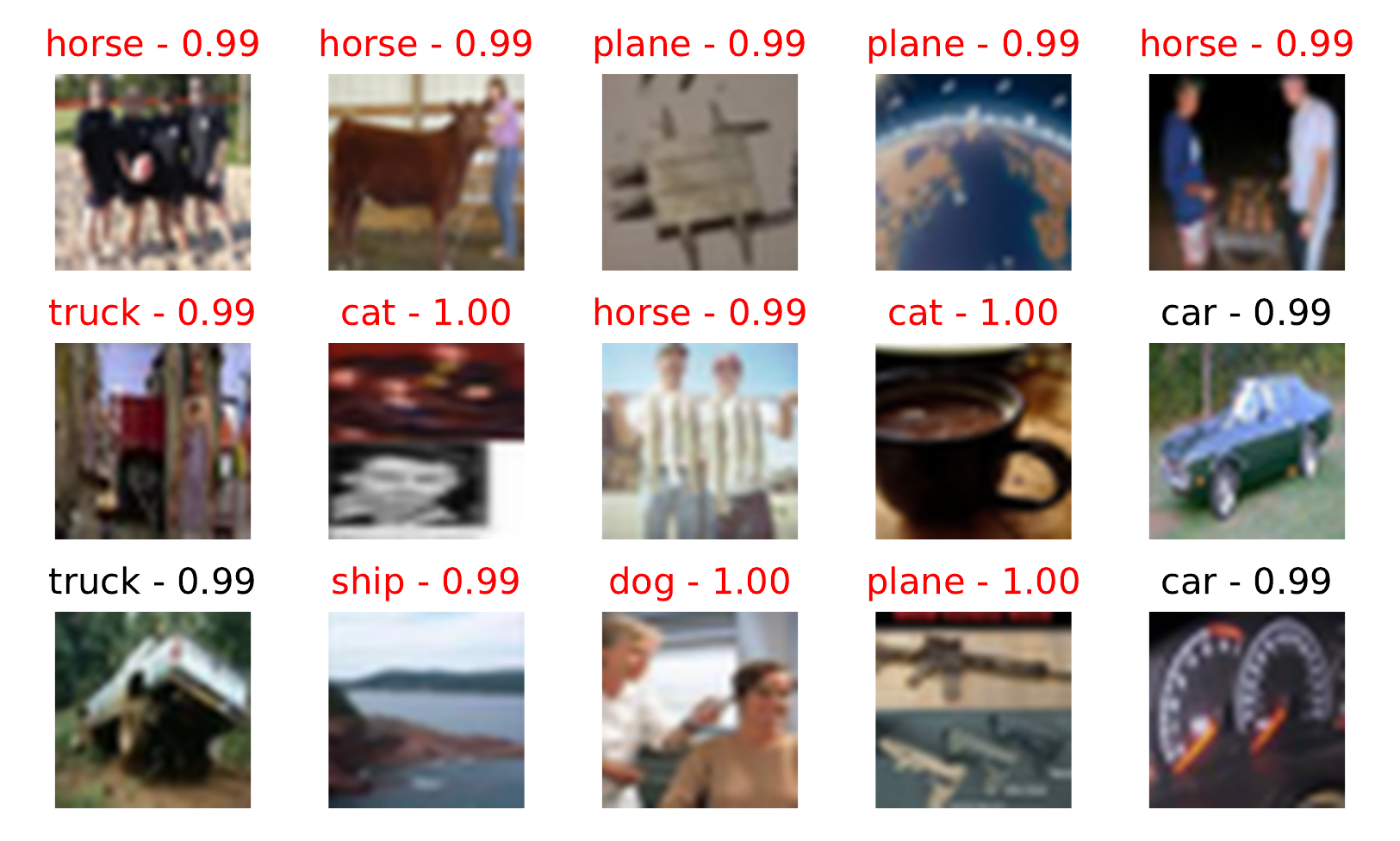}
}\\
\hline
\rotatebox{90}{\hspace{-0.3cm}\ours} &
\makecell{\includegraphics[width=\sc_img_width]{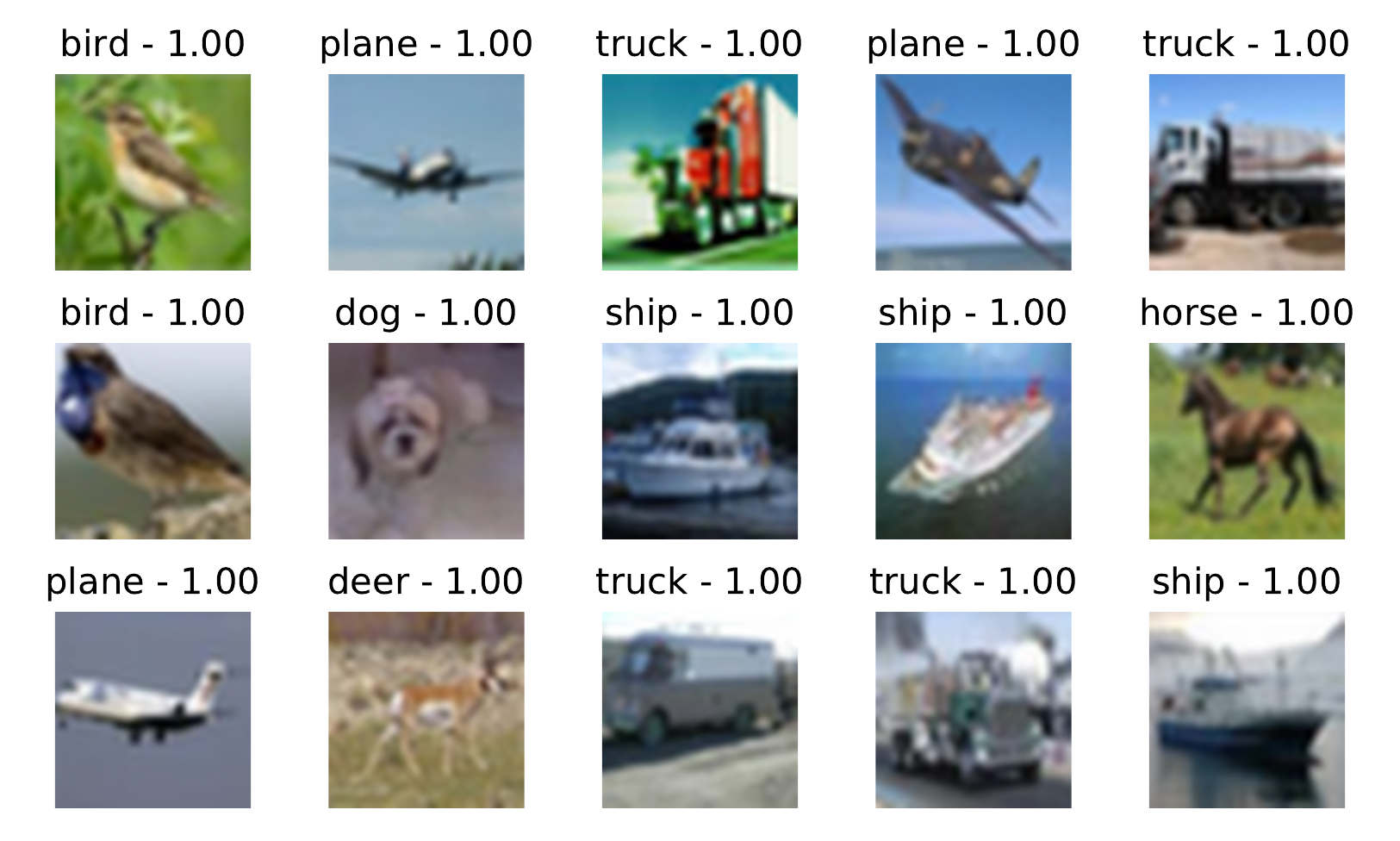}
} &
\makecell{\includegraphics[width=\sc_img_width]{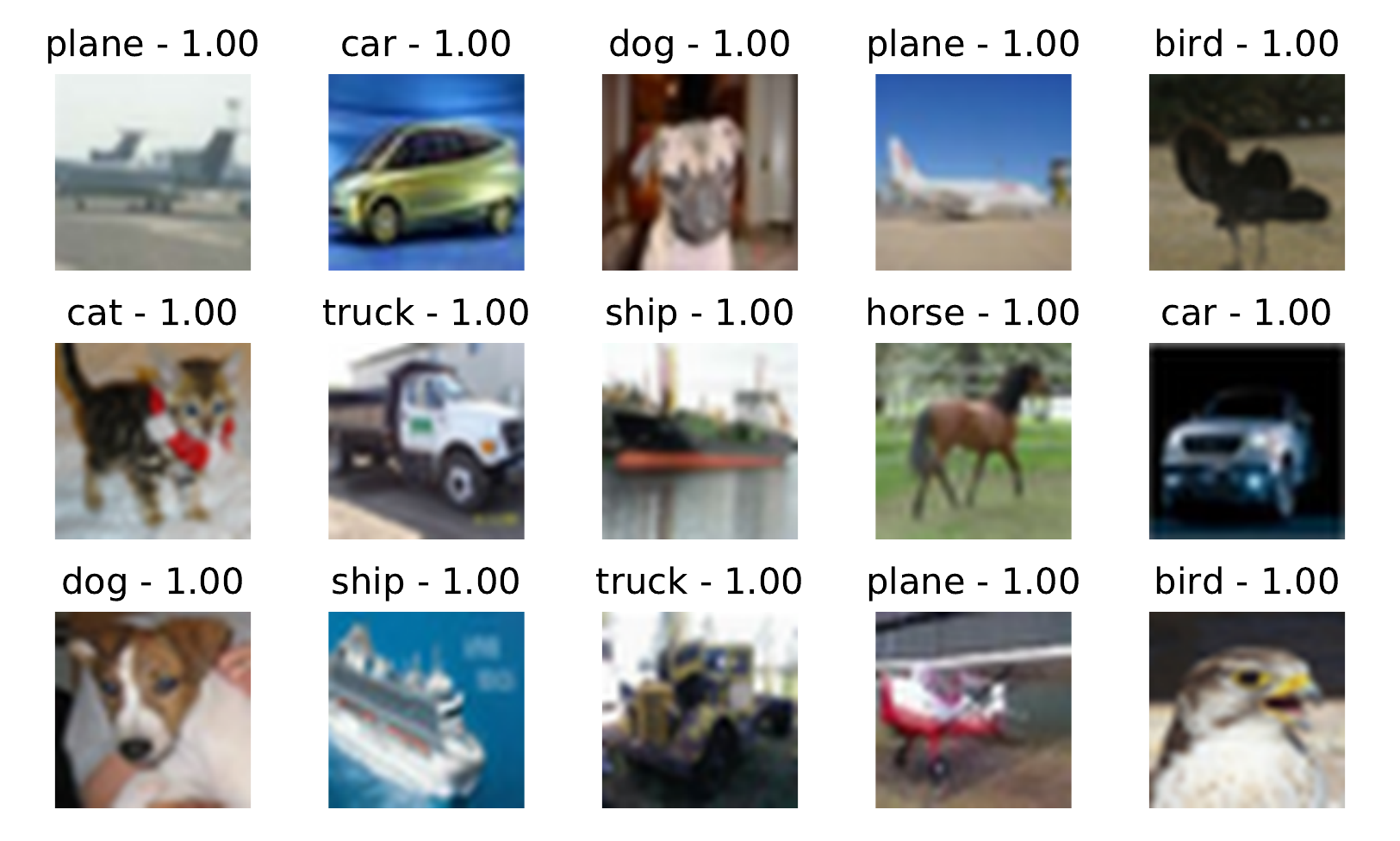}
} &
\makecell{\includegraphics[width=\sc_img_width]{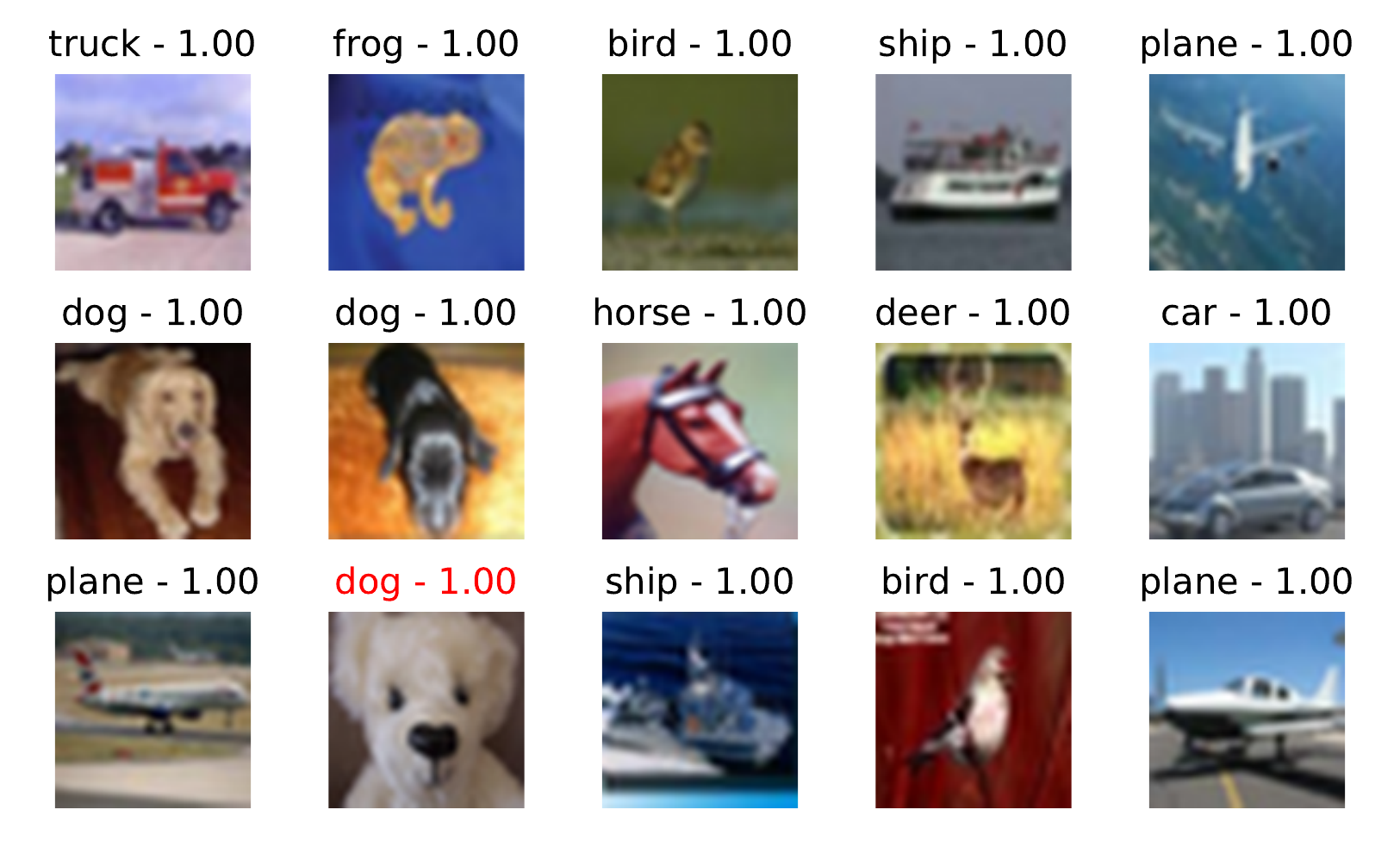}
}\\
\hline
\end{tabular}
\caption{CIFAR10 - random samples of exclusively selected samples from 80MTI for either \baseod (top) or \ours (bottom) over all three iterations of self-training. While \ours has good selection quality over all iterations, the selection of \baseod degrades even though it accepts less samples and shows a distribution shift: humans are classified as either `dog' or `horse'. }
\label{fig:cifar10_sample_demonstration}
\end{figure*}

%% file: figures/cifar100_samples.tex
\begin{figure*}[h!]
\small
\centering
\begin{tabular}{c|c}
    \hspace{-12.5pt}
    \makecell{
   \includegraphics[width=0.495\textwidth]{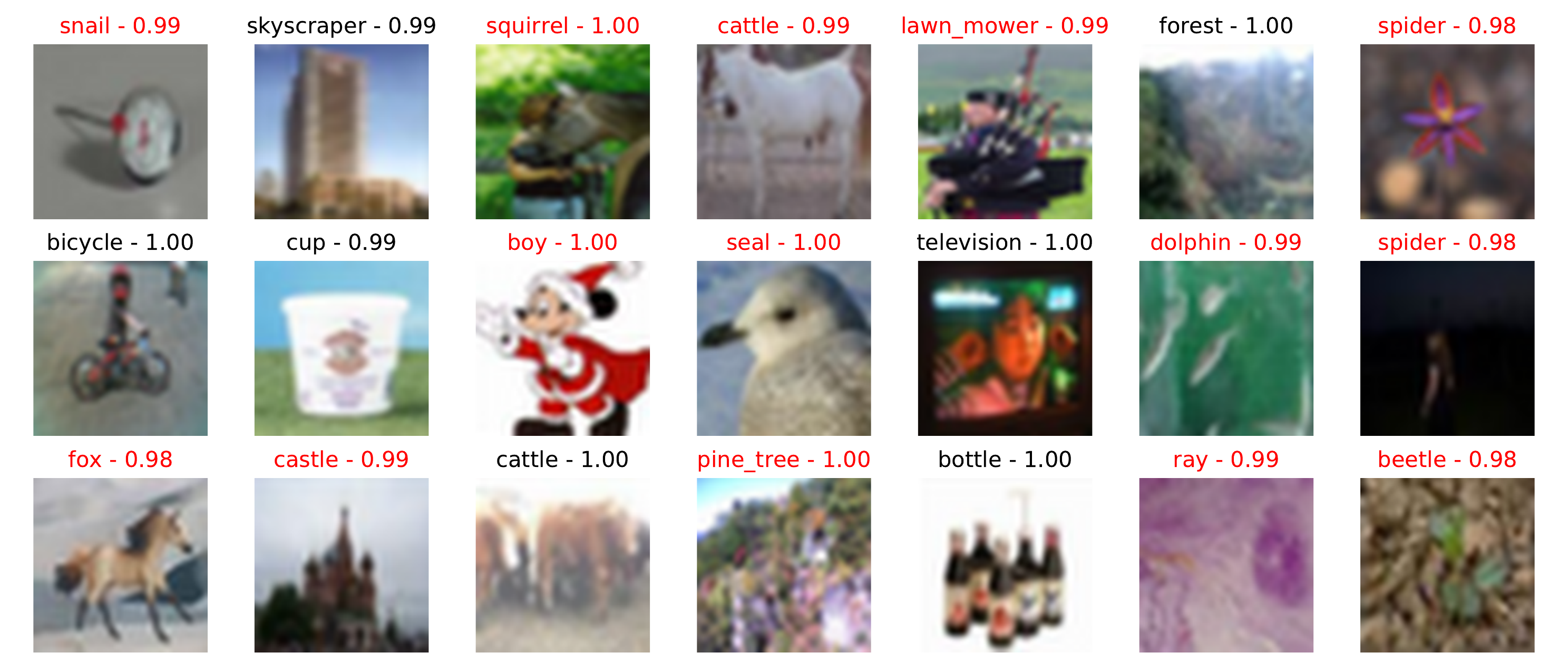}}
   &
   \makecell{
   \includegraphics[width=0.495\textwidth]{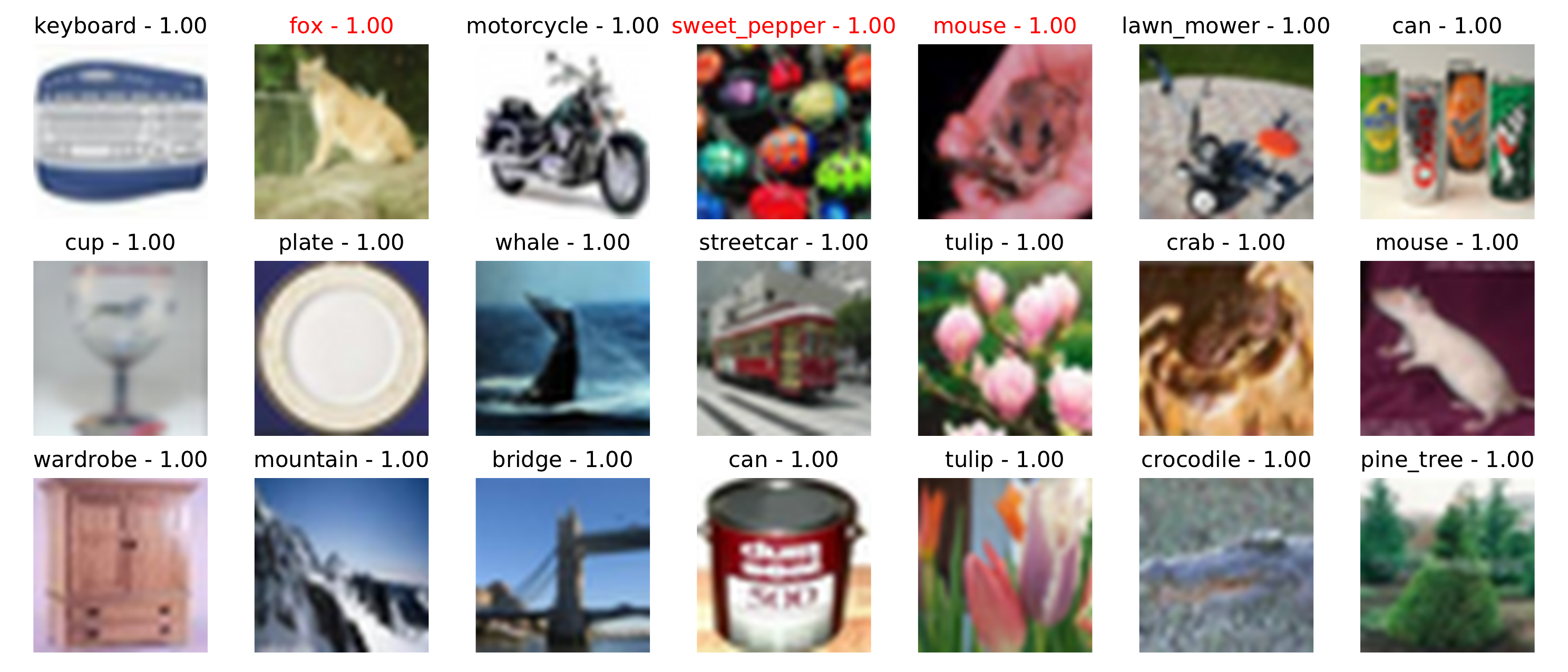}}\\
   \baseod & \ours
\end{tabular}
\caption{\label{fig:cifar100_samples}CIFAR100 - Random selection of 24 samples that are exclusively selected by either \baseod (left) or \ours (right) in the third iteration of self-training for the ResNet50 architecture. False positives are marked in red. }
\end{figure*}

%

%% file: tables/cifar100_accuracies.tex

\begin{figure*}[]
  \begin{minipage}[b]{0.74\textwidth}
    \centering
    \newcolumntype{F}{>{\centering\arraybackslash}m{1cm}}
    \newcolumntype{M}{>{\centering\arraybackslash}m{0.55cm}}
    \small
    \begin{tabular}{|F|*{4}{M}|*{4}{M}|*{4}{M}|}
    \hhline{~|----|----|----|}
    \multicolumn{1}{c|}{} & \multicolumn{4}{c|}{CIFAR100 error} & \multicolumn{4}{c|}{CIFAR100-C error} & \multicolumn{4}{c|}{OD-AUROC}\\
    \hline
    \textbf{ResNet} & Base & 1st & 2nd & 3rd & Base & 1st & 2nd & 3rd & Base & 1st & 2nd & 3rd\\
    \hhline{|=|====|====|====|}
    \ours & 20.02 & \emph{15.91} & \emph{15.38} & \textbf{14.86} & 43.39 & \emph{36.92} &  \emph{35.64} & \textbf{34.82} & \emph{91.76} & \textbf{93.42} & \emph{93.28} & \emph{92.49}\\
    \hline
    \base & \emph{19.31} & 17.56 & 17.54 & 18.40 & \emph{42.51} & 39.92 & 39.70 & 41.76 & 82.91 & 81.20 & 76.68 & 74.83  \\
    \hline
    \baseod & \emph{19.31} &17.97 & 17.84 & 18.76 & \emph{42.51} & 38.55 & 39.87 & 41.70 & 82.91 & 81.36 & 77.57 & 73.74\\
    \hhline{|=|====|====|====|}
    \textbf{Pyramid} & Base & 1st & 2nd & 3rd & Base & 1st & 2nd & 3rd & Base & 1st & 2nd & 3rd\\
    \hhline{|=|====|====|====|}
    \ours & 12.40 & \textbf{11.17} & \emph{11.34} & \emph{11.71} & 33.56 & \textbf{32.37} & \emph{32.97} & \emph{34.54} & \textbf{95.16} & \emph{94.91} & \emph{94.76} & \emph{94.45}\\
    \hline
    \base & \emph{11.93} & 12.29 & 13.30 & 14.74 & \emph{32.84} & 34.37 & 37.45 & 41.27 & 84.93 & 80.43 & 78.14 & 76.96\\
    \hline
    \baseod & \emph{11.93} & 12.47 & 13.39 & 14.51 & \emph{32.84} & 34.45 & 37.56 & 41.56 & 84.93 & 80.97 & 78.88 & 76.55\\
    \hline
    \end{tabular}
    \captionof{table}{\label{tab:CIFAR100}CIFAR100: 
    \base and \baseod do not improve test error over the base classifier. Only \ours improves consistently for ResNet50 over iterations ($5.16\%$ better) and can even improve performance for the Pyramid272 architecture ($1.23\%$ better). 
    }
    \label{tab:cifar100_error}
  \end{minipage}
  \hfill
  \begin{minipage}[b]{0.23\textwidth}
      \centering
      \includegraphics[width=1\textwidth]{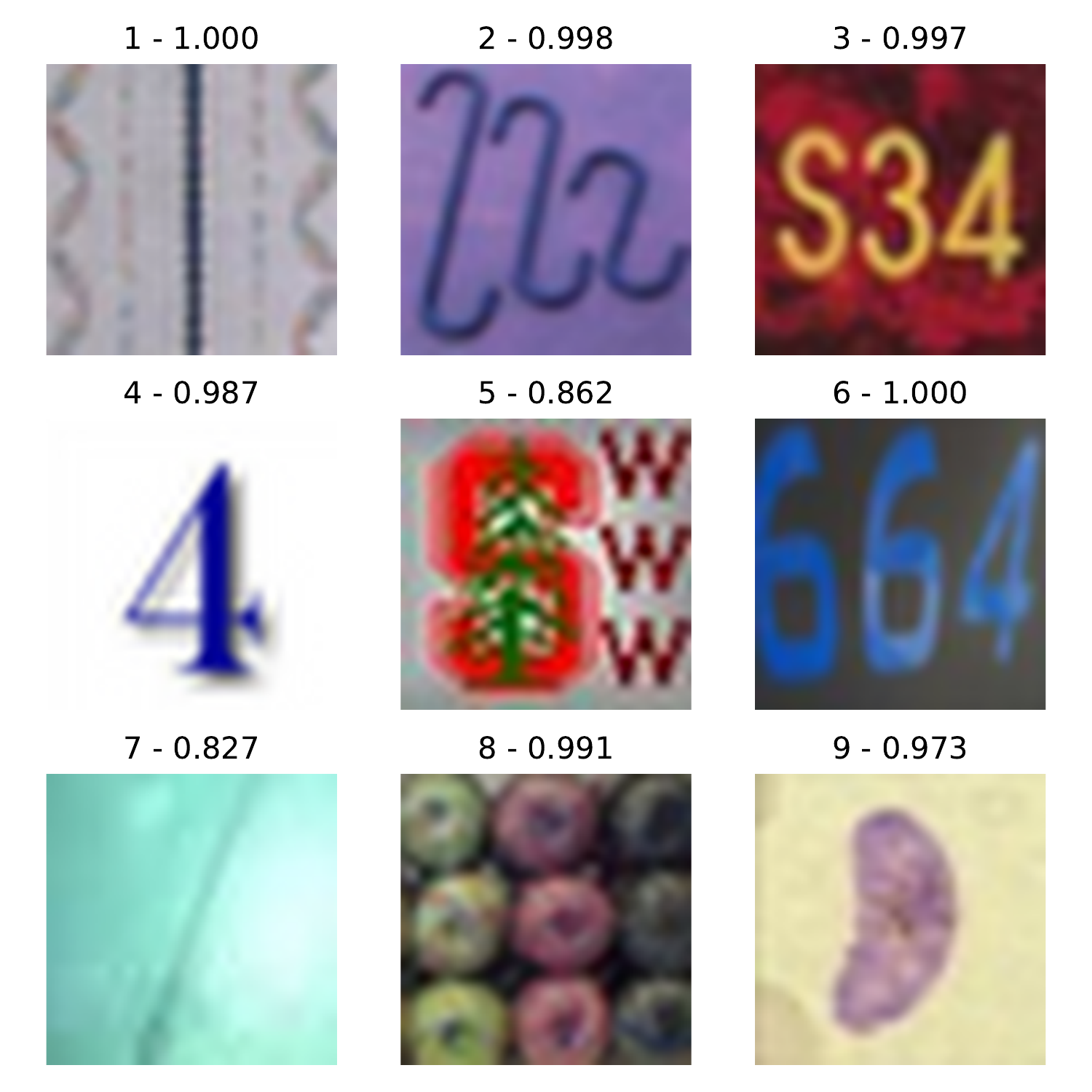}
      \captionof{figure}{\label{fig:SVHN-sample}SVHN: Samples selected by \ours from 80MTI show often digits.}
  \end{minipage}
\end{figure*}

%% file: tables/svhn_resnet_accuracy.tex
\begin{table}
    \centering
    \begin{tabular}{|c||c|c|c|c|}
        \hline 
        \textbf{SVHN-Error} & Base & 1st & 2nd & 3rd \\
        \hhline{|=#=|=|=|=|}
        \ours & \emph{1.94} & 1.55 & \emph{1.46} & \textbf{1.36}\\
        non-SVHN sel. & & 0.4\% & 1.6\% & 5.9\%\\
        \hhline{|=#=|=|=|=|}
        \base-CW & 2.03 & \emph{1.54} & 1.53 & 1.52\\ 
        \hline
        Supervised Train+521k & \multicolumn{4}{c|}{1.24}\\
        \hline
    \end{tabular}
    \caption{\label{tab:svhn_resnet_error}SVHN: test error of \ours in open world (other row: fraction of selected 80MTI samples) vs. self-training \base-CW in closed world setting and supervised training.}
\end{table}


%% file: figures/app_duplicate_hist.tex
\begin{figure*}[h!]
\small
\centering
\begin{subfigure}{0.49\textwidth}
\includegraphics[width=\textwidth]{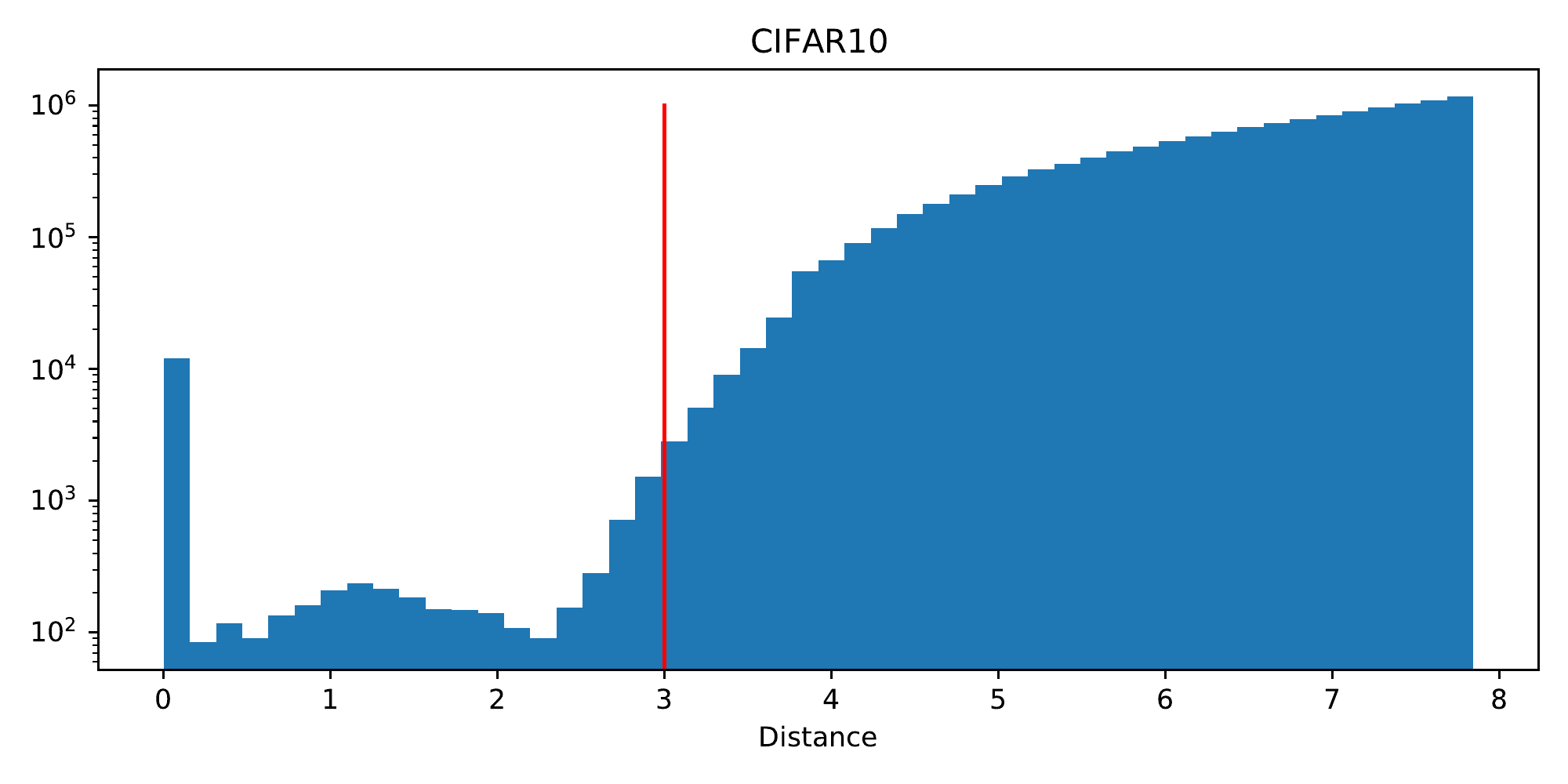}
\end{subfigure}
\begin{subfigure}{0.49\textwidth}
\includegraphics[width=\textwidth]{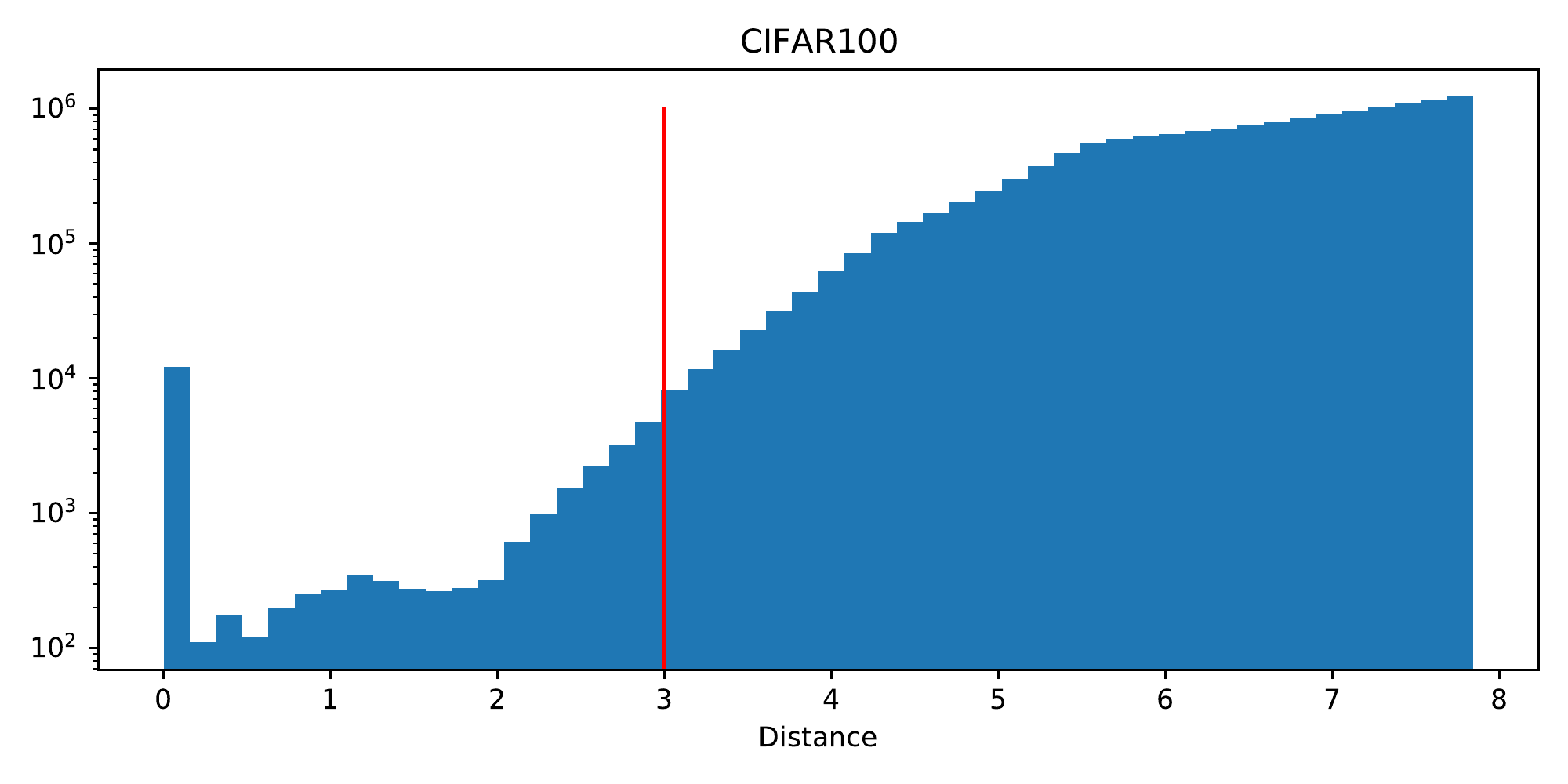}
\end{subfigure}

\caption{Logarithmic histogram of nearest neighbour distances between CIFAR test set and 80MTI for image pairs with an $l_2$ distance below $2000/255$. }
\label{fig:appendix_duplicate_hist}
\end{figure*}

%% file: figures/app_duplicate_removal.tex
\begin{figure*}[h!]
\begin{subfigure}{\textwidth}
\includegraphics[width=\textwidth]{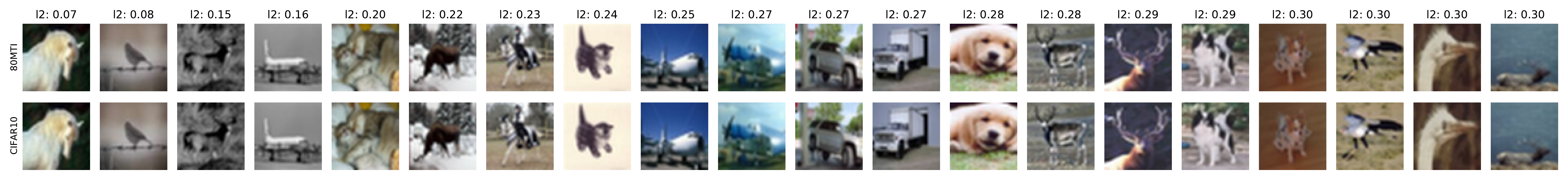}
\hrule
\includegraphics[width=\textwidth]{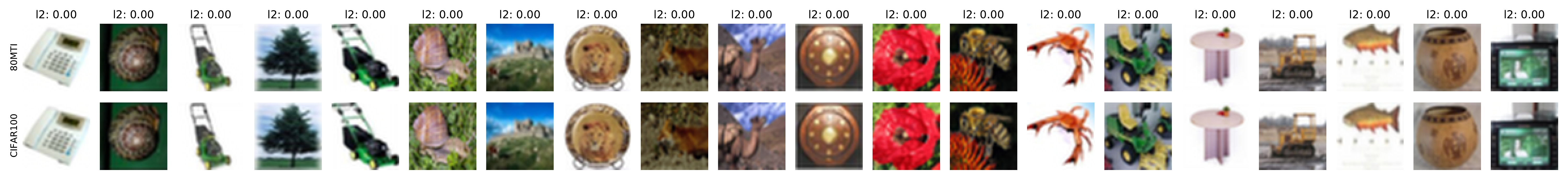}
\caption{80MTI images after \cite{HenMazDie2019}'s duplicate removal with their nearest CIFAR10 (upper) and CIFAR100 (bottom) test set neighbours  sorted by $l_2$-distance. Note that 
there are still CIFAR10 test set near-duplicates with distance larger than $0$ and exact CIFAR100 test set duplicates.}
\label{fig:appendix_duplicates_hendrycks}
\end{subfigure}

\begin{subfigure}{\textwidth}
\includegraphics[width=\textwidth]{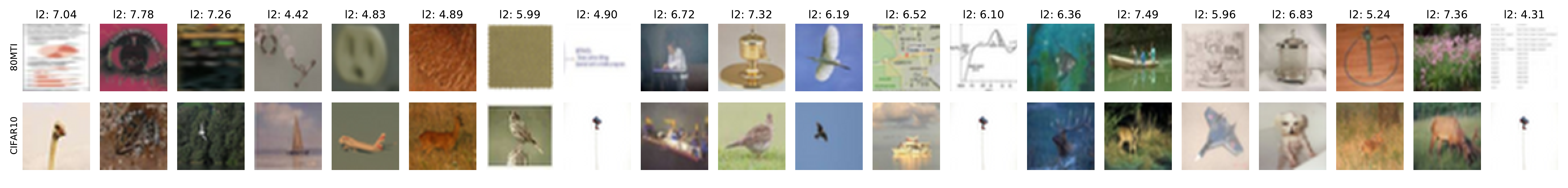}
\hrule
\includegraphics[width=\textwidth]{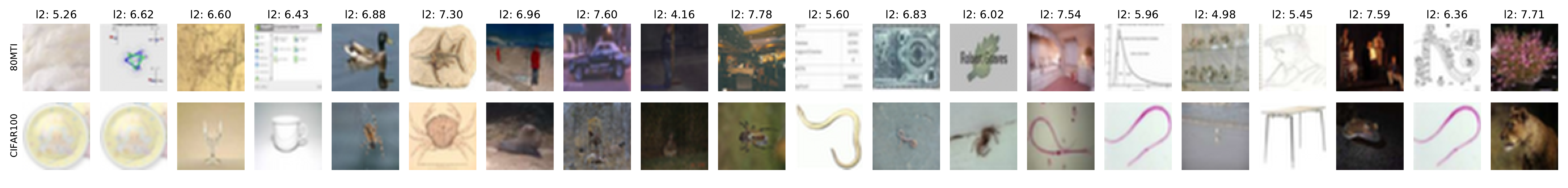}
\caption{Random selection of 80MTI samples with their respective CIFAR $l_2$-nearest neighbour for sample pairs with distance smaller than $2000/255 \approx 7.84$. This threshold was used by  \cite{CarEtAl19} for duplicate removal but is too aggressive at it removes too many unrelated images.}
\label{fig:appendix_duplicates_cifar_784}
\end{subfigure}

\begin{subfigure}{\textwidth}
\includegraphics[width=\textwidth]{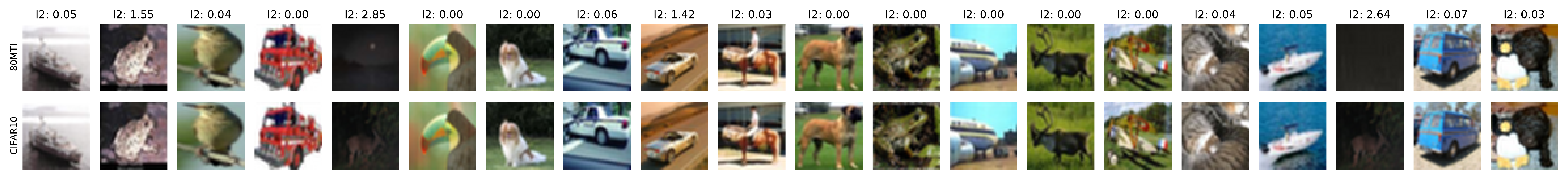}
\hrule
\includegraphics[width=\textwidth]{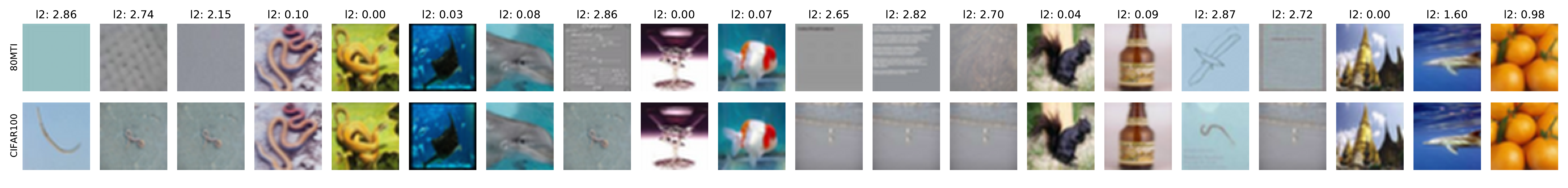}
\caption{80MTI samples with their CIFAR $l_2$-nearest neighbour for sample pairs with $l_2$-distance $\leq 3.0$ (random selection). Most  pairs are near duplicates and are thus removed in our duplicate removal. Even for this small $l_2$ radius, we find false positives for monochrome images.}
\label{fig:appendix_duplicates_cifar_3}
\end{subfigure}

\begin{subfigure}{\textwidth}
\includegraphics[width=\textwidth]{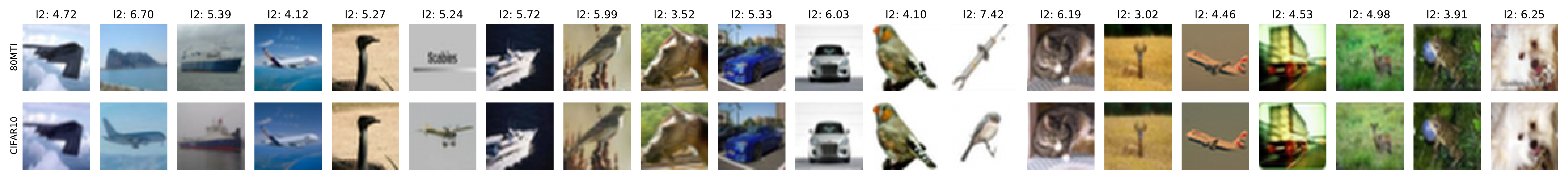}
\hrule
\includegraphics[width=\textwidth]{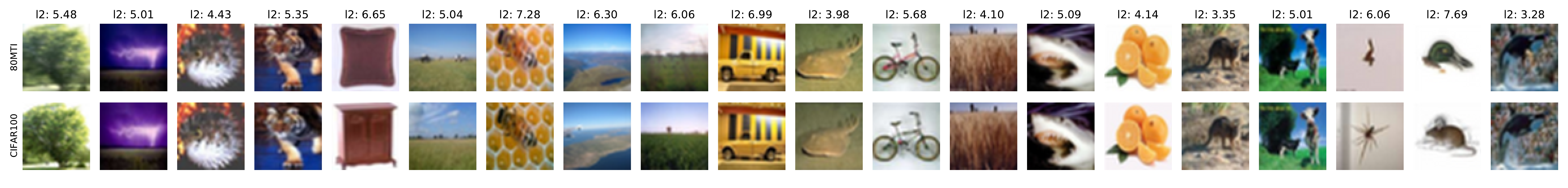}
\caption{Random selection of samples with $l_2$-distance in $[3, 2000/255]$ that are marked as duplicates wrt to both the LPIPS and SSIM threshold. We are able to find mostly duplicates even in regimes where the $l_2$ distance starts to become meaningless. In total our duplicate removal has a low number of false positives while we could not find any false negative (CIFAR test set image in 80MTI after duplicate removal).}
\label{fig:appendix_duplicates_cifar10_784}
\end{subfigure}

\caption{Visualization of various exclusion thresholds for CIFAR10 (above line) and CIFAR100 (below line). The top image shows the sample from 80MTI and the lower one the nearest neighbour in the test set.}
\end{figure*}

%% file: tables/app_num_samples.tex
\begin{table*}[t]
\begin{tabular}{|l||rrrrrrrrrr|}
\hline
\ours & plane  & car    & bird   & cat   & deer  & dog   & frog  & horse  & ship   & truck \\
\hhline{|=#==========|}
1st  & 65851  & 345636 & 99529  & 25820 & 13735 & 41894 & 8308  & 151406 & 117886 & 32152 \\
2nd  & 94198  & 357582 & 108365 & 48432 & 15054 & 44242 & 9161  & 167501 & 114301 & 35032 \\
3rd  & 102979 & 332256 & 97470  & 52006 & 17141 & 59476 & 11576 & 134020 & 116660 & 39271\\
\hhline{|=#==========|}
\base & plane   & car    & bird    & cat    & deer    & dog     & frog    & horse   & ship    & truck  \\
\hhline{|=#==========|}
1st  & 408410  & 786267 & 991242  & 259478 & 295440  & 533440  & 211276  & 641187  & 1112572 & 10697  \\
2nd  & 1201883 & 903382 & 2141996 & 220140 & 626600  & 988225  & 1220826 & 1461518 & 443693  & 335166 \\
3rd  & 797443  & 311856 & 658163  & 474090 & 1199681 & 1342559 & 556252  & 1425692 & 1018332 & 542613\\
\hhline{|=#==========|}
\baseod & plane & car    & bird  & cat   & deer  & dog   & frog  & horse & ship  & truck \\
\hhline{|=#==========|}
1st                    & 83628 & 109902 & 59482 & 30161 & 13813 & 52180 & 6840  & 78036 & 86010 & 10697 \\
2nd                    & 39588 & 21111  & 35017 & 23175 & 19778 & 52114 & 11397 & 68127 & 57483 & 12709 \\
3rd                    & 30546 & 42453  & 11460 & 52185 & 8800  & 50571 & 3302  & 80201 & 12629 & 16842\\
\hline
\end{tabular}
\caption{CIFAR10: Number of samples above the in- and out-distribution thresholds for \ours and \baseod respectively above the in-distribution threshold for \base. Note that the number of samples above the threshold can be larger than the maximal number of accepted samples, which increases from 25k over 50k to 75k per class in iteration 3.}
\label{tab:app_num_samples_cifar10}
\end{table*}

%% file: figures/app_cifar100_samples_histogram.tex
\begin{figure*}
\centering
\includegraphics[width=\textwidth]{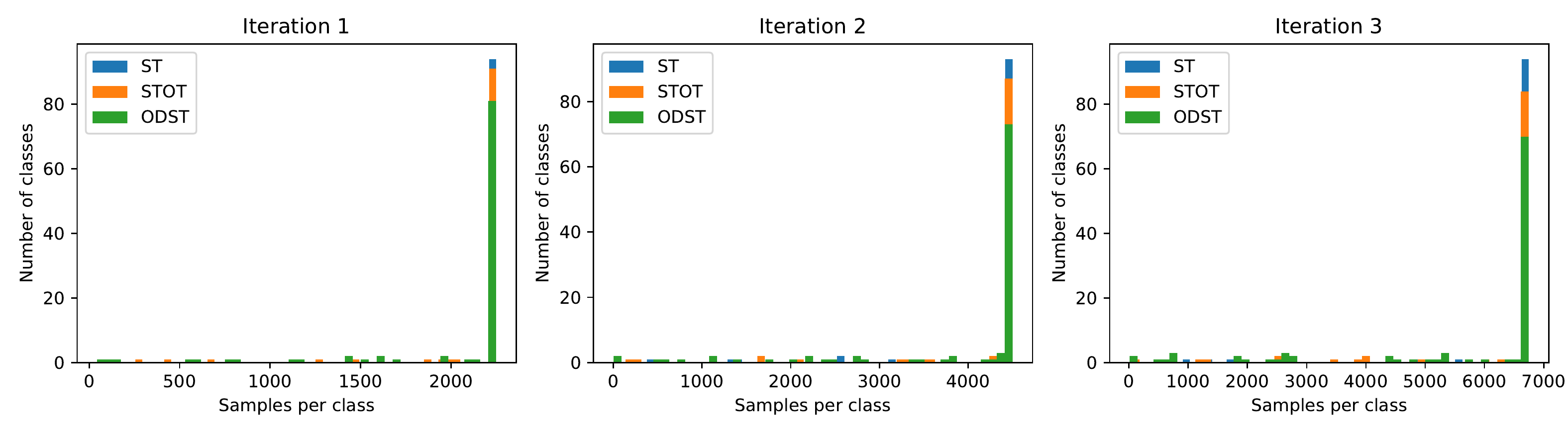}
\caption{CIFAR100: Histogram plot showing the number of accepted samples per class for \ours, \base and \baseod over three iterations. Each methods adds up to 2250 samples per class in the first, 4500 in the second and 6750 in the third iteration (note that the number of labeled examples is 450 per class). We see that all methods add the full number of samples for the majority of classes but \ours has the most conservative but also the best sample selection.}
\label{fig:appendix_cifar100_samples_hist}
\end{figure*}

%% file: tables/ablation_cifar_non_iterative.tex
\begin{table}[]
    \centering
    \begin{tabular}{|c|c|c|}
        \hline
        \textbf{CIFAR10} & \ours 3rd 75k& Non-iterative 75k\\
        \hline
        Test error & \textbf{1.88} & 1.97\\
        \hline
        CIFAR10.1 & 4.55 & \textbf{4.45}\\
        \hline
        CIFAR10-C & 15.02 & \textbf{13.53}\\
        \hline
        OD-AUROC & 98.98 & \textbf{99.31}\\
        \hhline{===}
        \textbf{CIFAR100} & \ours 3rd 6.75k& Non-iterative 6.75k\\
        \hline
         Test error & \textbf{14.86} & 16.15\\
        \hline
        CIFAR100-C & \textbf{34.82} & 36.53\\
        \hline
        OD-AUROC & 92.49 & \textbf{93.32}\\
        \hline
    \end{tabular}
    \caption{Comparison of the third iteration \ours model, including up to 15 times the samples per class and a non-iterative \ours model trained directly with the same number of samples selected by the base teacher. Our iterative self-training outperforms the one-shot self-training in terms of test error significantly for CIFAR100 and a bit for CIFAR10.}
    \label{tab:ablation_cifar10_non_iterative}
\end{table}

%% file: tables/ablation_cifar_cominbed.tex
\begin{table*}[]
    \small
    \centering
    \begin{tabular}{|c|c|ccc|ccc|ccc|}
     \hhline{~|-|---|---|---|}
     \multicolumn{1}{c|}{} & Base & \multicolumn{3}{c|}{\ours} &   \multicolumn{3}{c|}{Hard labels $\U \setminus \I$} & \multicolumn{3}{c|}{No label smoothing}\\
     \hline
     \textbf{CIFAR10} & 0th & 1st & 2nd & 3rd & 1st & 2nd & 3rd & 1st & 2nd & 3rd\\
     \hline
     Test error & 3.19 & 2.15 & 2.01 & \textbf{1.88} & 2.52 & 2.83 & 2.73 & 2.18 & 2.22 & 2.40\\
     \hline
     CIFAR10.1 & 7.00 & 5.40 & \textbf{4.55} & 4.70 & 6.40 & 6.65 & 6.60 & 4.90 & 4.75 & 5.40\\
     \hline
     CIFAR10-C & 16.86 & 14.19 & 15.02 & 15.46 & 15.39 & 14.49 & \textbf{13.96} & 14.62 & 15.15 & 16.00\\
     \hline 
     OD-AUROC & 98.95 &\textbf{99.20} & 99.16 & 98.98 & 98.96 & 98.90 & 99.03 & 99.10 & 98.72 & 98.35\\
     \hhline{===========}
     \textbf{CIFAR100} & 0th & 1st & 2nd & 3rd & 1st & 2nd & 3rd & 1st & 2nd & 3rd\\
     \hline
     Test error & 20.02 & 15.91 & 15.38 & \textbf{14.86} & 18.27 & 17.64 & 17.77 & 16.13 & 15.31 & 15.58\\
     \hline
     CIFAR100-C & 43.39 & 36.92 & 35.64 & \textbf{34.82} & 38.85 & 37.75 & 36.80 & 36.49 & 35.58 & 35.69\\
     \hline
     OD-AUROC & 91.76 & 93.42 & 93.28 & 92.49 & \textbf{93.62} & 93.39 & 92.29 & 93.00 & 93.16 & 92.17\\
     \hline
    \end{tabular}
    \caption{Comparison of different choices for the labels of non-selected images in $\U \setminus \I$. \ours uses soft-labels with label smoothing factor 0.5, hard labels enforces $1/K$ on all samples and no label-smoothing is equivalent to \ours but without any smoothing. All models are trained using the same base model.}
    \label{tab:ablation_cifar_combined}
\end{table*}

%% file: figures/app_cifa10_samples.tex
\begin{figure*}[h!]
\small
\centering
\begin{subfigure}{0.33\textwidth}
\includegraphics[width=\textwidth]{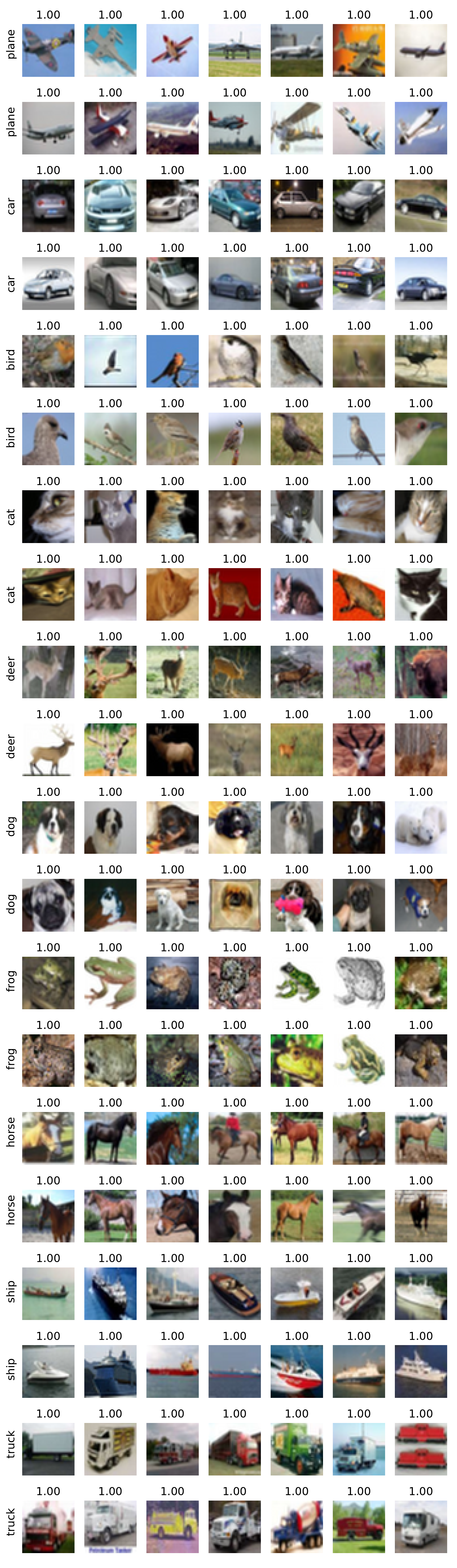}
\caption{1. iteration - max. 25k samples per class}
\end{subfigure}
\begin{subfigure}{0.33\textwidth}
\includegraphics[width=\textwidth]{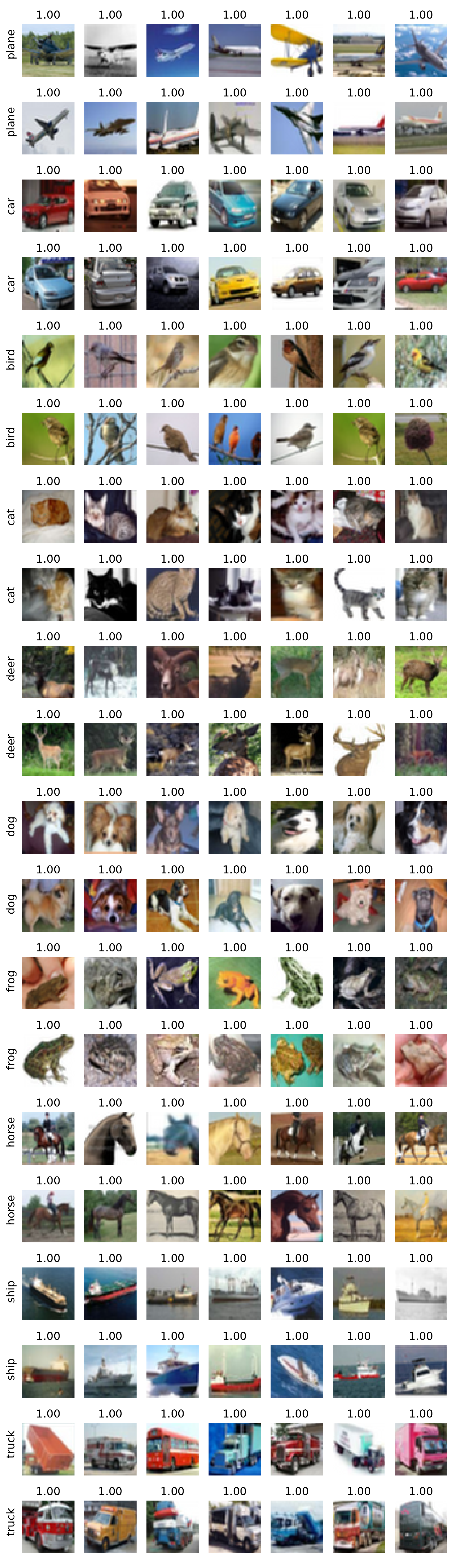}
\caption{2. iteration - max. 50k samples per class}
\end{subfigure}
\begin{subfigure}{0.33\textwidth}
\includegraphics[width=\textwidth]{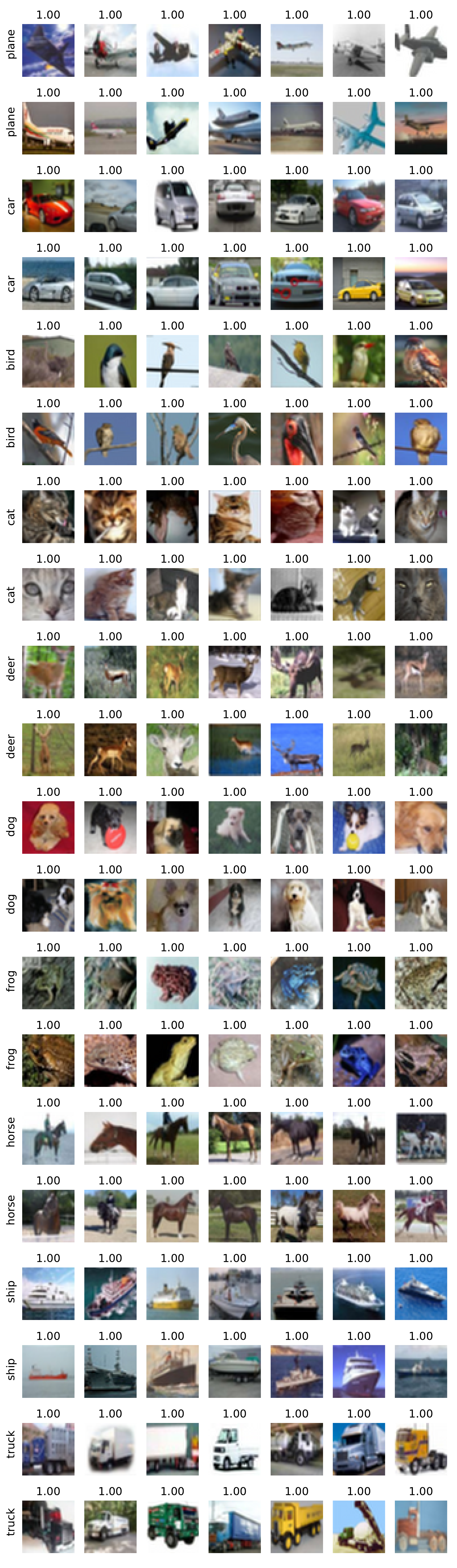}
\caption{3. iteration - max. 75k samples per class}
\end{subfigure}
\caption{CIFAR10: \ours selected samples for the ResNet50 architecture over iterations. Sample selection remains stable over iterations, even with the addition of up to 75k samples per class per iteration. This is an astonishingly good result, given the difficulty of the task due to the tiny fraction of CIFAR10 related images in 80MTI. Note that \ours has much better sample selection accuracy even though it selects more samples than \baseod, see Table \ref{tab:app_num_samples_cifar10}.}
\label{fig:appendix_cifar10_odst}
\end{figure*}

\begin{figure*}[h!]
\small
\centering
\begin{subfigure}{0.33\textwidth}
\includegraphics[width=\textwidth]{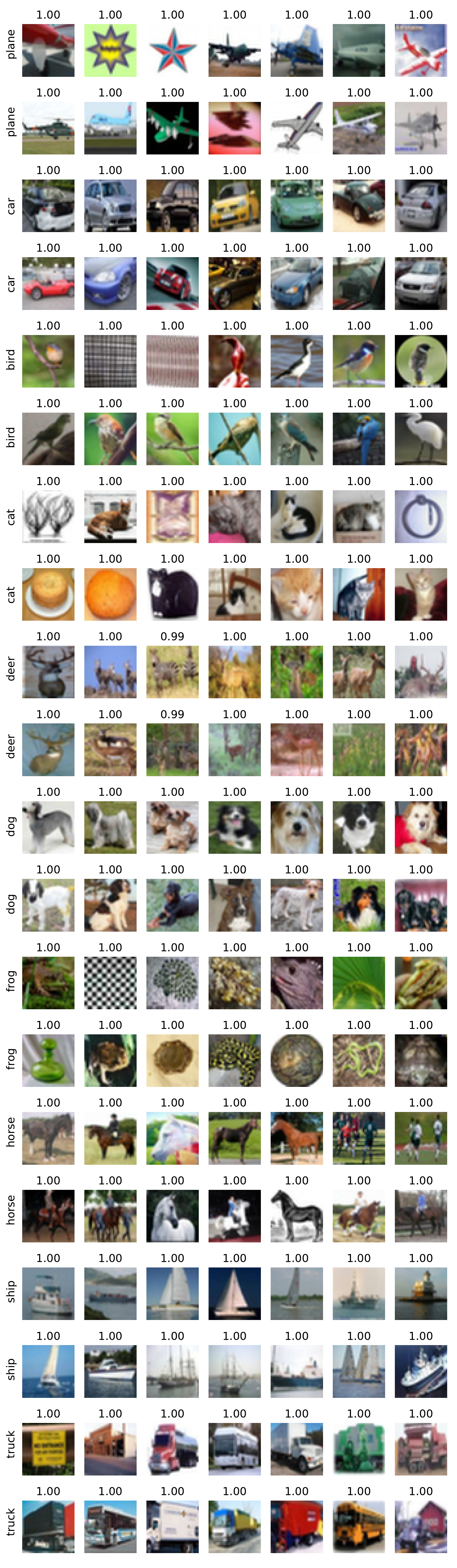}
\caption{1. iteration - max. 25k samples per class}
\end{subfigure}
\begin{subfigure}{0.33\textwidth}
\includegraphics[width=\textwidth]{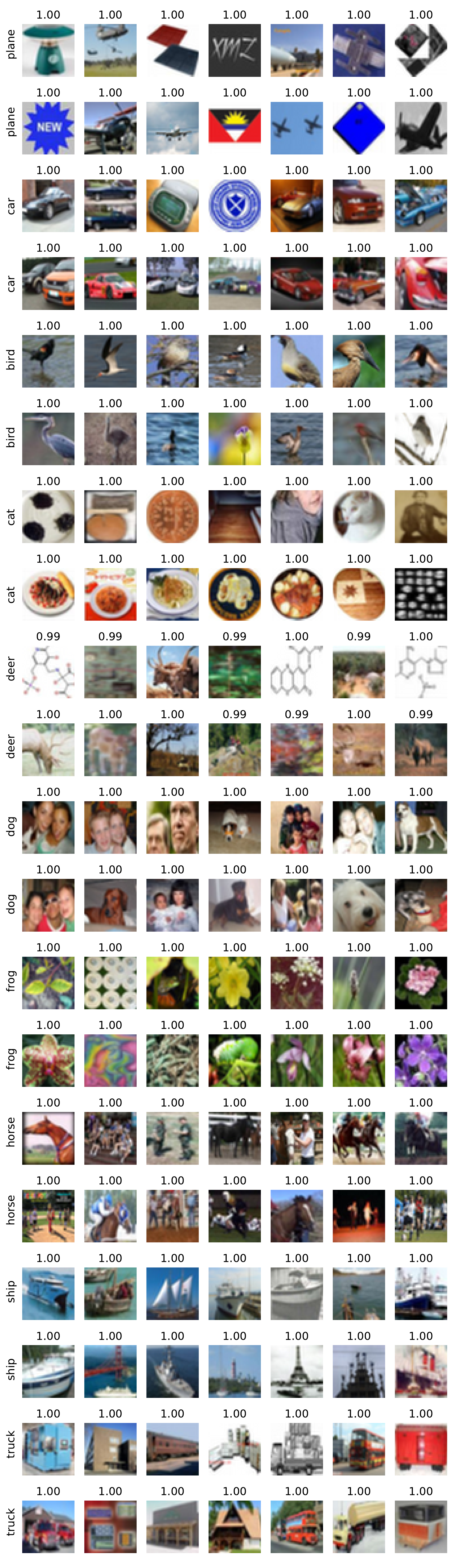}
\caption{2. iteration - max. 50k samples per class}
\end{subfigure}
\begin{subfigure}{0.33\textwidth}
\includegraphics[width=\textwidth]{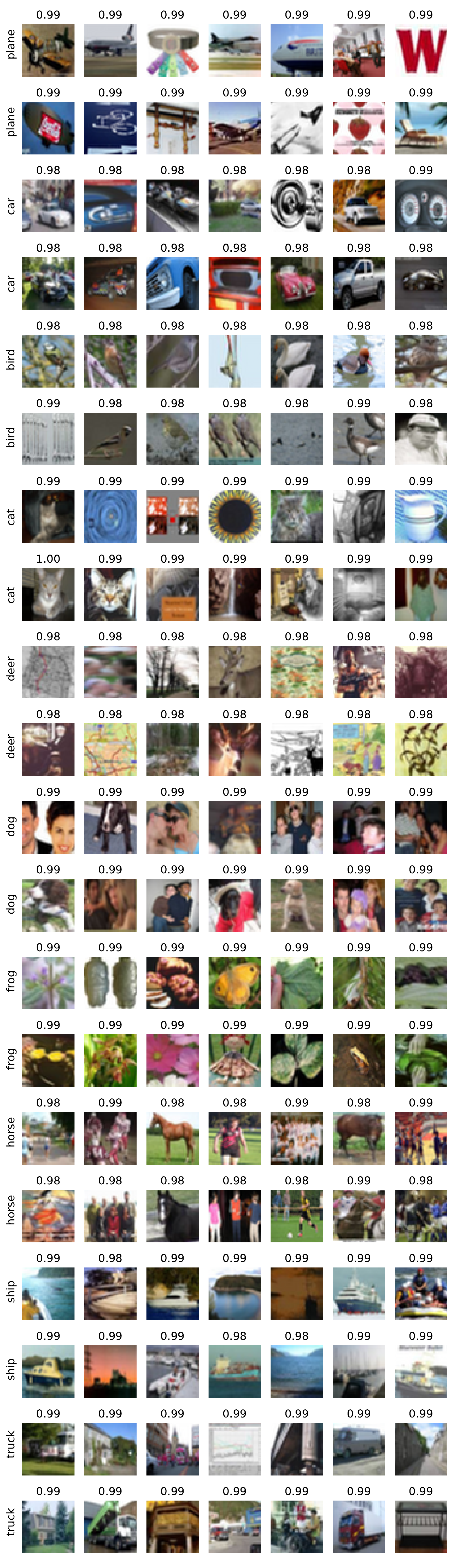}
\caption{3. iteration - max. 75k samples per class}
\end{subfigure}\vspace{-2mm}
\caption{CIFAR10: selected samples by \base for the ResNet50 architecture over iterations. Without OD-thresholding, even the first teacher model selects a lot of unrelated samples. The later students learn from 
these mistakes and \eg start to associate humans with "dog" and "horse".  This highlights the difficulty of iterative self-training as the student has no way to recover from the failures of the teacher. 
As (c) maintains $3.54\%$ test error this shows that test error alone is not a good indicator if the classifier has learned good task representations.
}
\label{fig:appendix_cifar10_st}
\end{figure*}

\begin{figure*}[h!]
\small
\centering
\begin{subfigure}{0.33\textwidth}
\includegraphics[width=\textwidth]{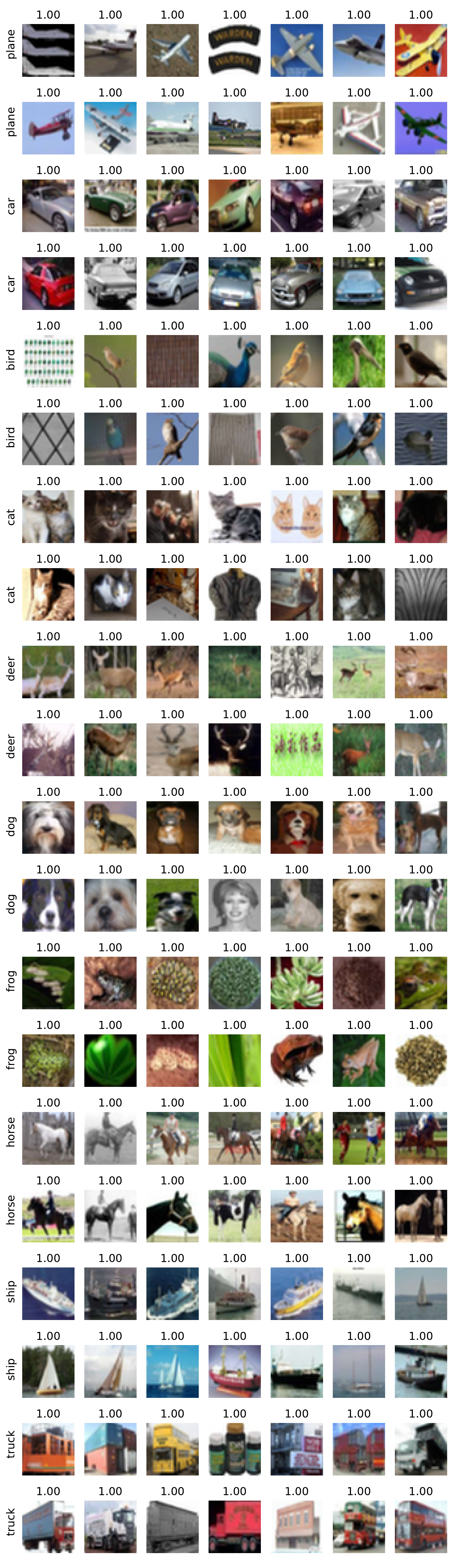}
\caption{1. iteration - max. 25k samples per class}
\end{subfigure}
\begin{subfigure}{0.33\textwidth}
\includegraphics[width=\textwidth]{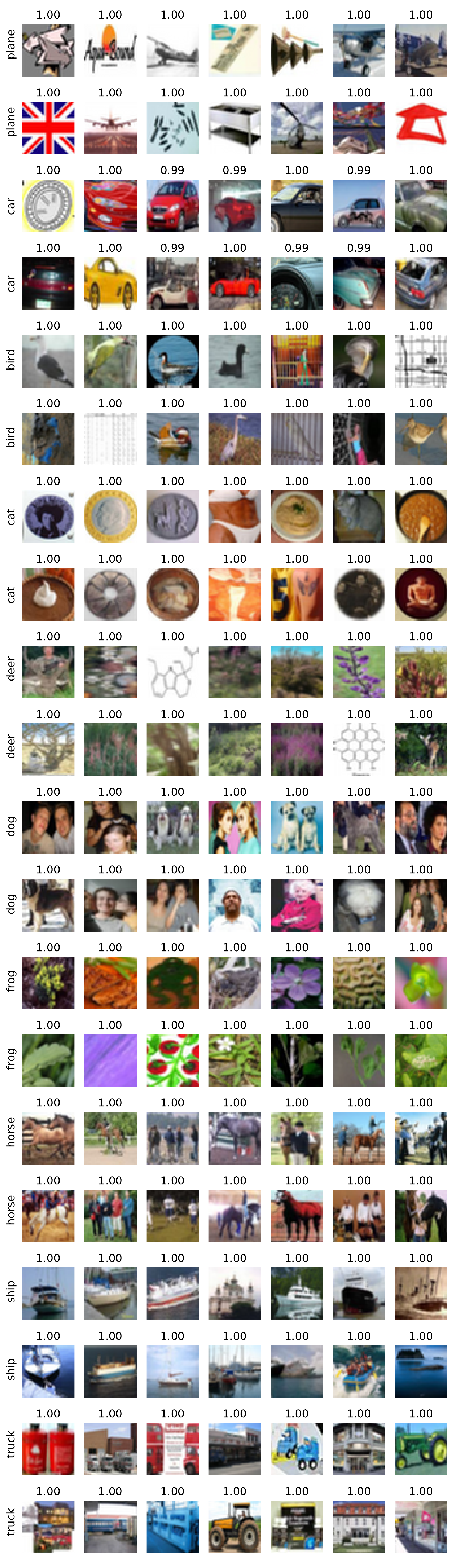}
\caption{2. iteration - max. 50k samples per class}
\end{subfigure}
\begin{subfigure}{0.33\textwidth}
\includegraphics[width=\textwidth]{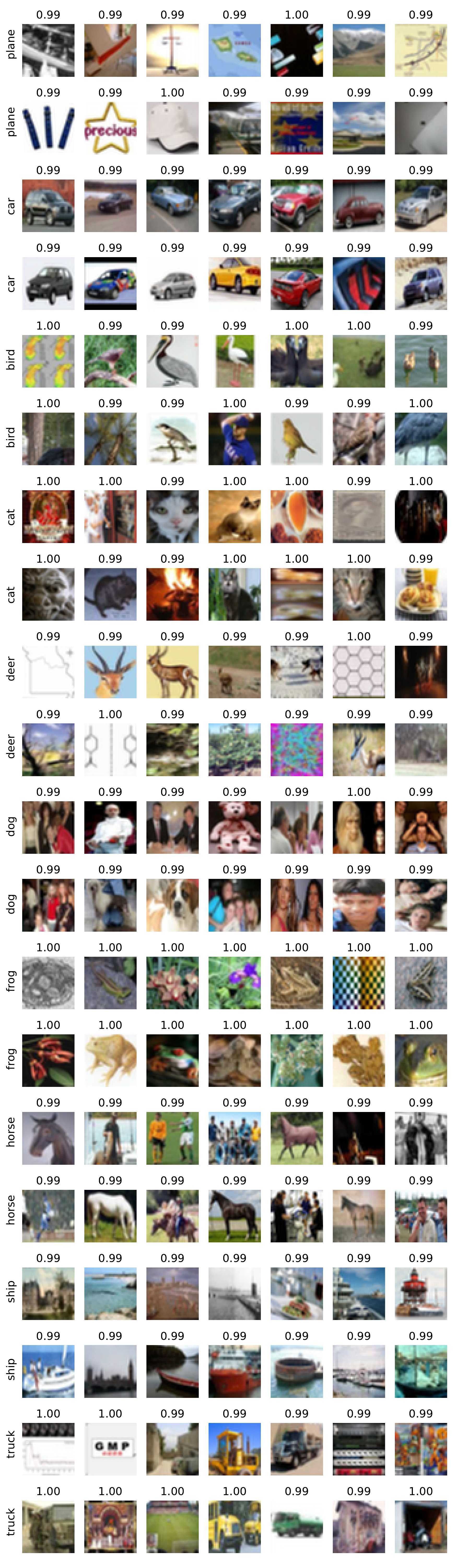}
\caption{3. iteration - max. 75k samples per class}
\end{subfigure}
\caption{CIFAR10: selected samples by  \baseod for the ResNet50 architecture over iterations. With the addition of OD-thresholding, \baseod selects more useful samples in the first iteration. However the error accumulates and the later models associate plates with "cat", chemical structures with "deer" and humans with "dog". This shows that both OD-thresholding \emph{and} OD-aware training are required for iterative open world self-training in order to learn proper class representations.}
\label{fig:appendix_cifar10_st_ot}
\end{figure*}

\begin{figure*}[h!]
\small
\centering
\begin{subfigure}{0.48\textwidth}
\includegraphics[width=\textwidth]{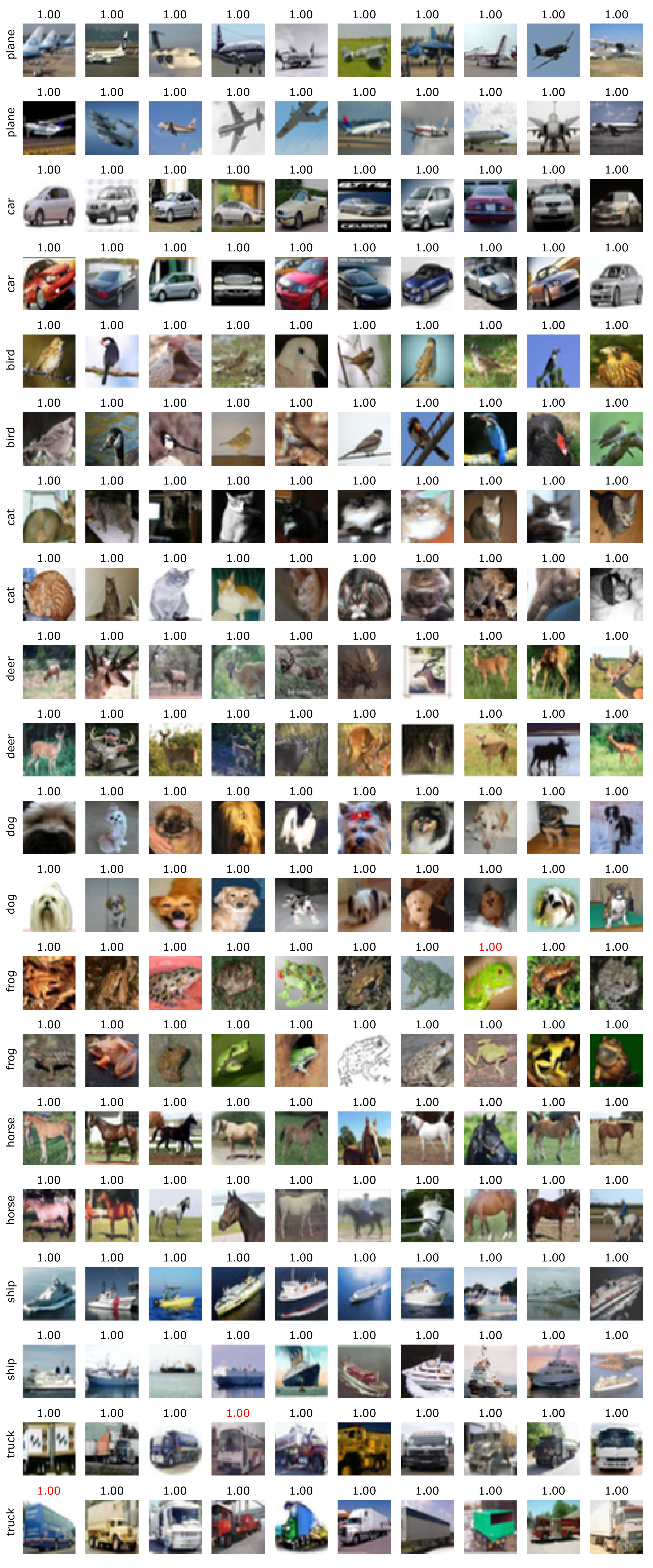}
\caption{\ours - 2nd iteration - 50k samples per class}
\end{subfigure}
\begin{subfigure}{0.48\textwidth}
\includegraphics[width=\textwidth]{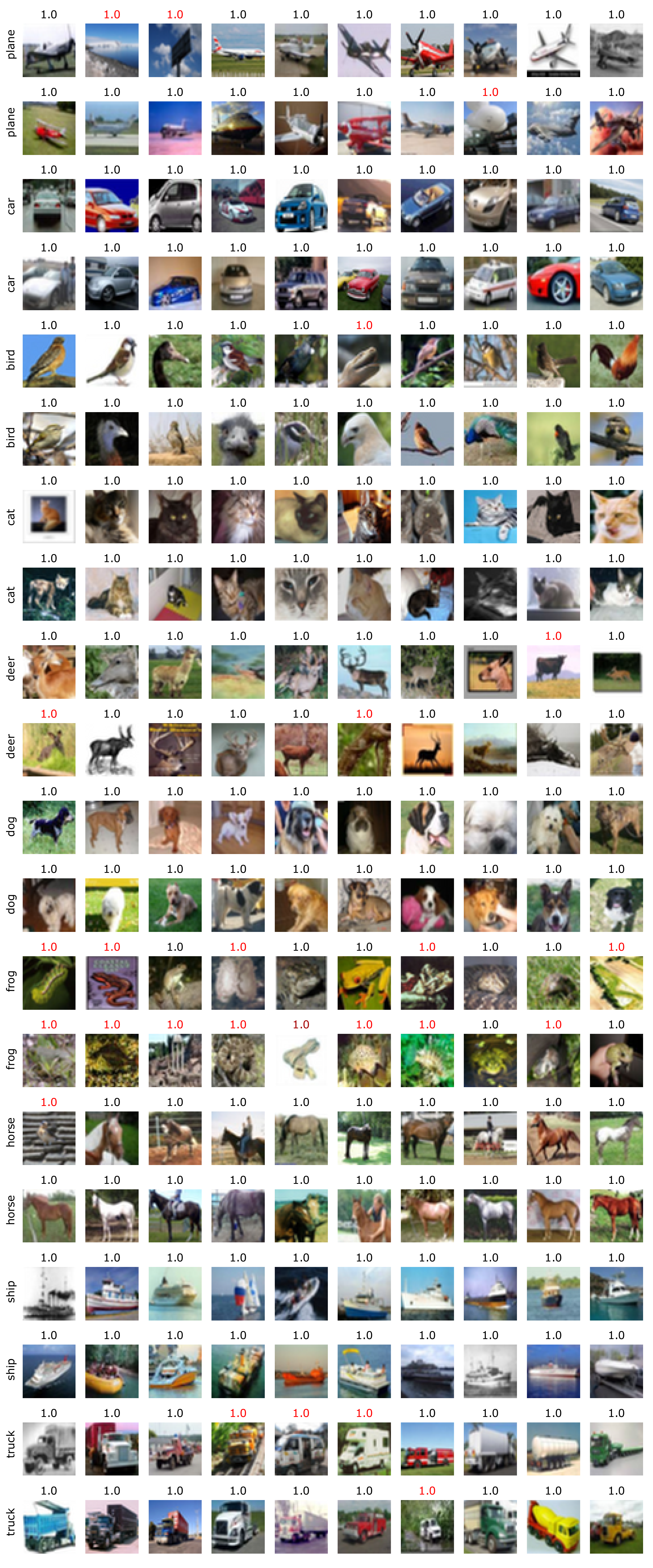}
\caption{500k-TI \cite{CarEtAl19} - 50k samples per class}
\end{subfigure}
\caption{CIFAR10: Comparison of \ours to \cite{CarEtAl19}. Obvious misclassifications are marked in red. Especially for the classes "frog" and "deer", \ours has a way more accurate sample selection due to our thresholding whereas \cite{CarEtAl19} include unrelated samples. Additionally, they include some false positives for "plane" and both include related but wrong objects like trains and (mini)busses for "truck".}
\label{fig:appendix_cifar10_carmon}
\end{figure*}

%% file: figures/app_cifar100_samples.tex
\begin{figure*}[h!]
\centering
\includegraphics[width=\textwidth]{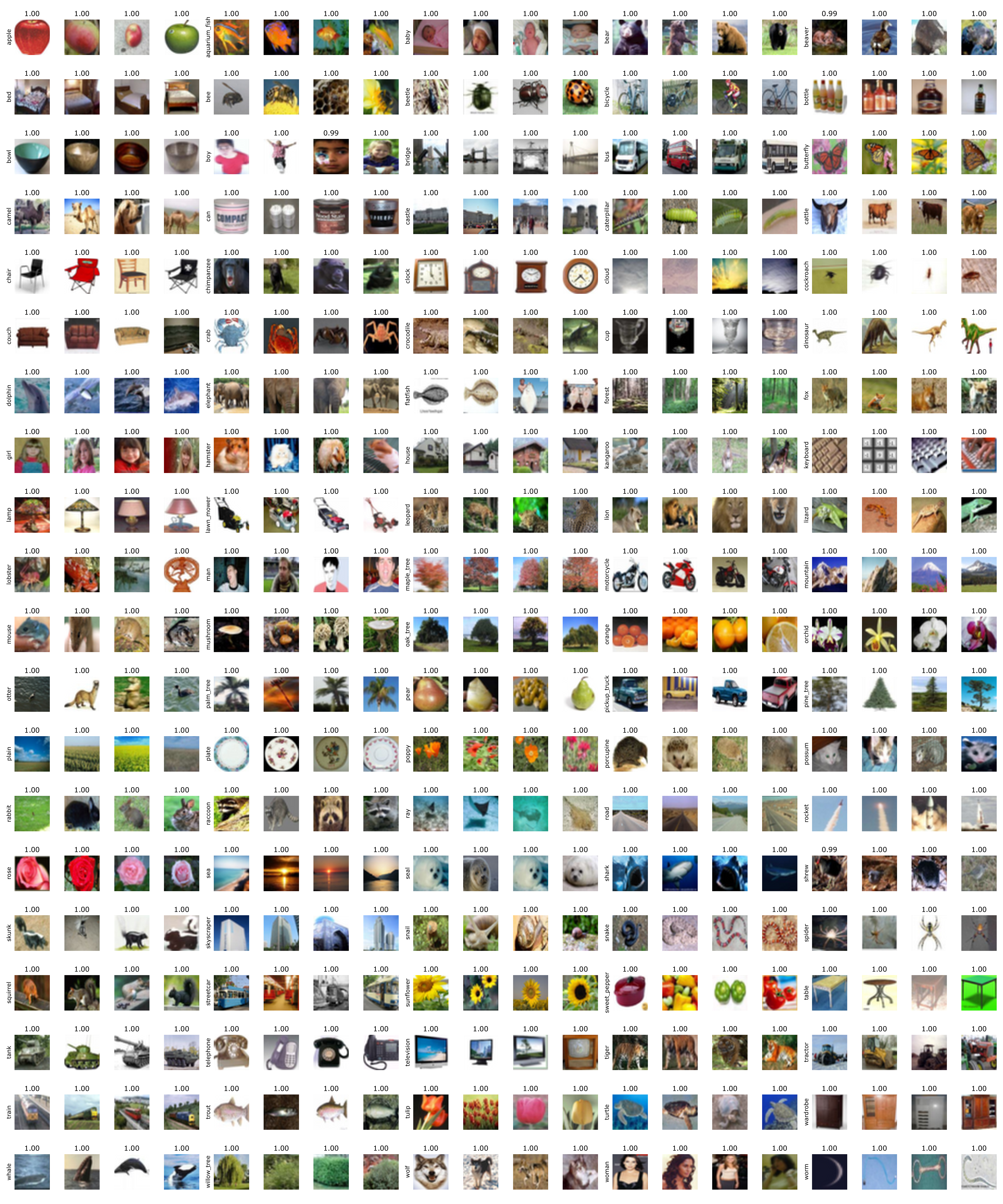}
\caption{CIFAR100: \ours 1st iteration (up to 2.25k randomly selected samples per class). Even with 100 classes in total and only 450 train images per class, \ours is able to select a diverse range of task-related images from the pool of 80MTI, which contains mostly unrelated samples.}
\label{fig:appendix_cifar100_odst_1}
\end{figure*}

\begin{figure*}[h!]
\centering
\includegraphics[width=\textwidth]{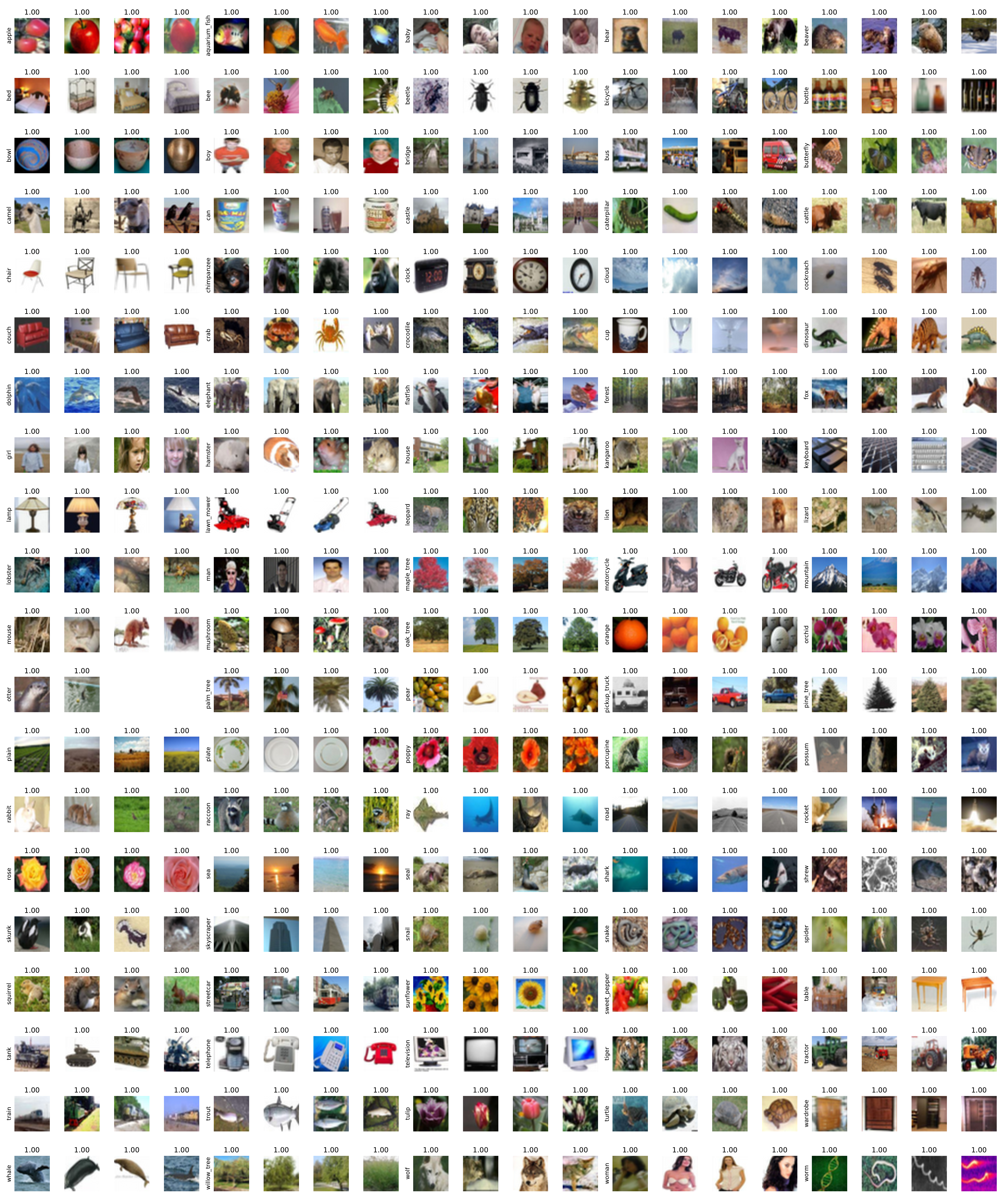}
\caption{CIFAR100: \ours 2nd iteration (up to 4.5k randomly selected samples per class). With \ours, class representations remain stable as the student model is not learning to classify unrelated images with high confidence. } 
\label{fig:appendix_cifar100_odst_2}
\end{figure*}

\begin{figure*}[h!]
\centering
\includegraphics[width=\textwidth]{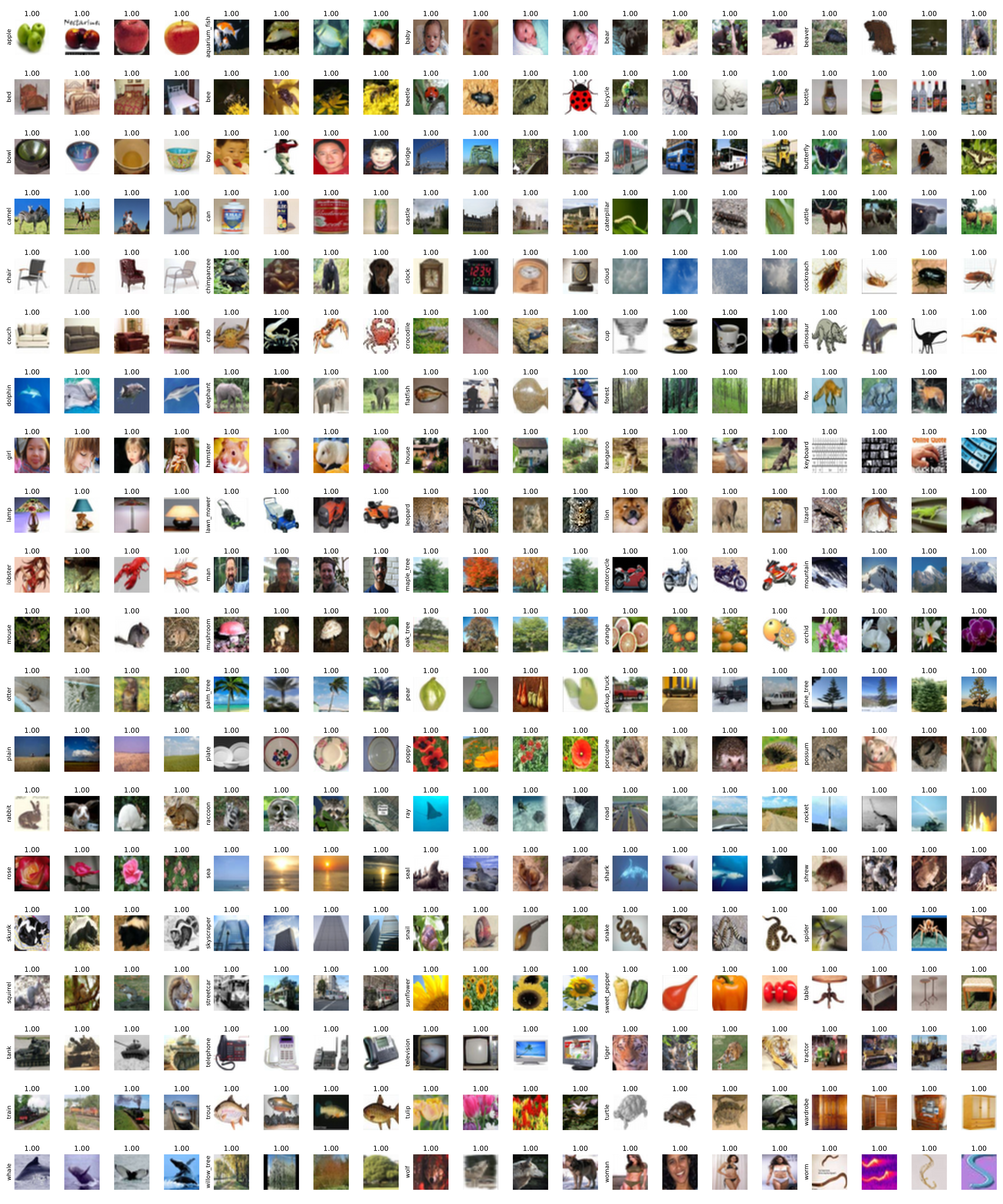}
\caption{CIFAR100: \ours 3nd iteration (up to 6.75k randomly selected samples per class). Taking into account the difficulty of the task \ours shows great stability even in the third iteration, however sample quality depends on the class and some classes like "otter" and "worm" show a larger false positive rate. For the class otter there are only two images above the thresholds. } 
\label{fig:appendix_cifar100_odst_3}
\end{figure*}

\begin{figure*}[h!]
\centering
\includegraphics[width=\textwidth]{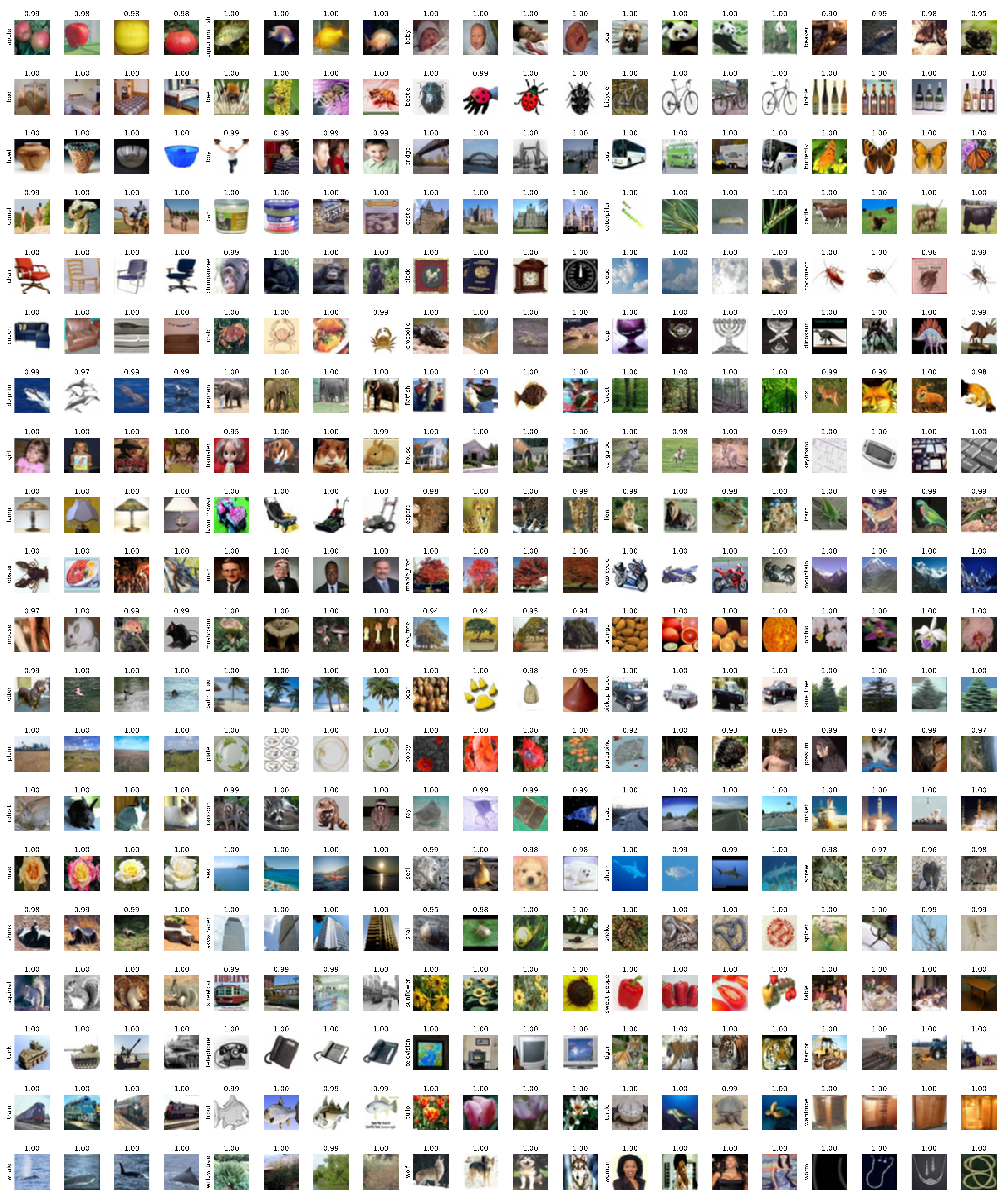}
\caption{CIFAR100: \base 1st iteration (up to 2.25k randomly selected samples per class). Due to the relatively small number of additional samples per class in the first iteration, even the baseline is able to select comparatively good sampls.}
\label{fig:appendix_cifar100_st_1}
\end{figure*}

\begin{figure*}[h!]
\centering
\includegraphics[width=\textwidth]{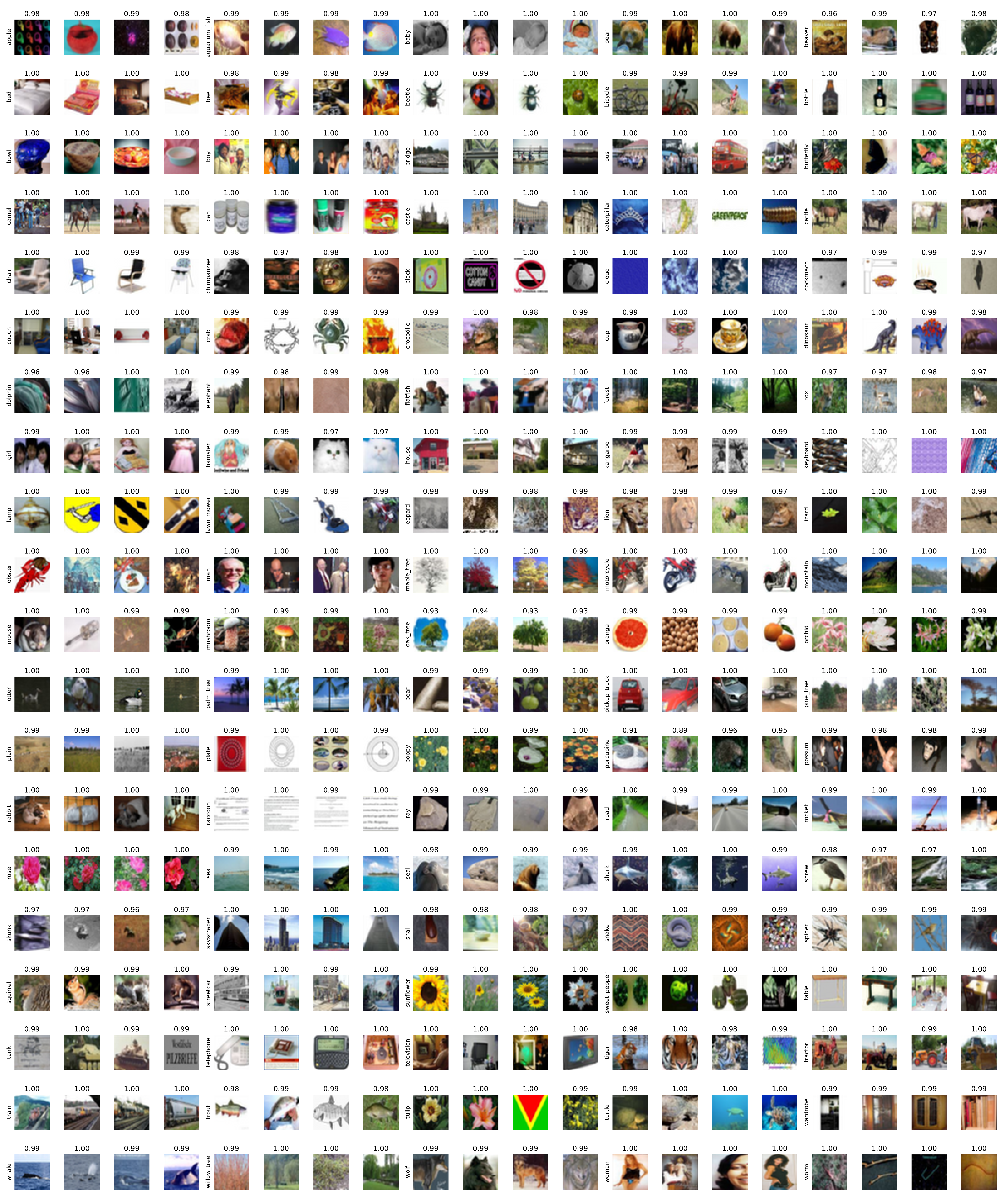}
\caption{CIFAR100: \base 2nd iteration (up to 4.5k randomly selected samples per class). With the 2nd iteration we can notice an increase in task unrelated samples across various classes. As most images are classified with very large confidence, the additional samples are not well suited for training another student.  }
\label{fig:appendix_cifar100_st_2}
\end{figure*}

\begin{figure*}[h!]
\centering
\includegraphics[width=\textwidth]{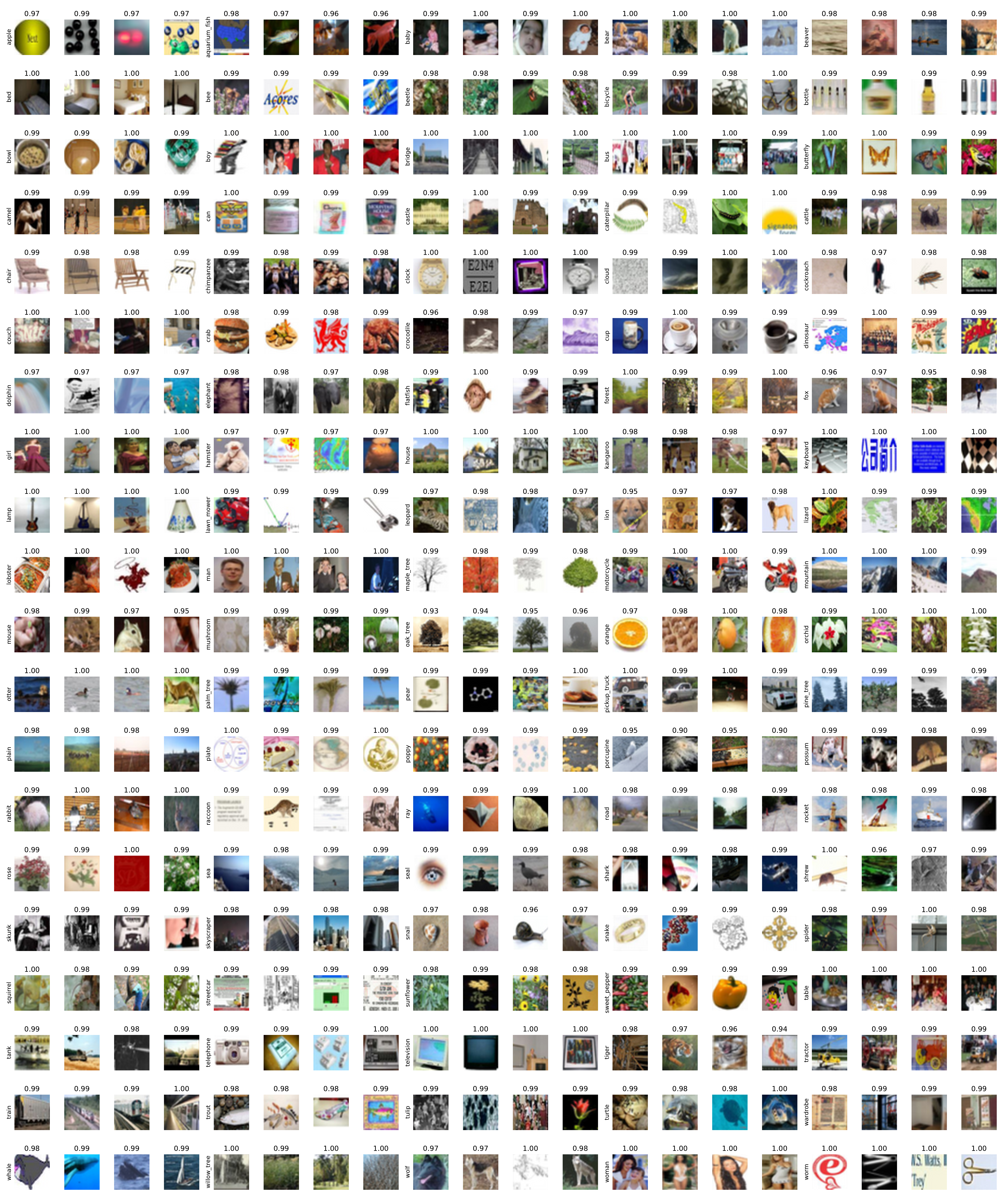}
\caption{CIFAR100: \base 3nd iteration (up to 6.75k randomly selected samples per class). For most classes, standard self-training breaks down at this point. Notice that while the first model was selecting mostly good images and had an intact class representations, this one has learned systematically wrong representations, \eg labeling food as "crab" with high confidence.}
\label{fig:appendix_cifar100_st_3}
\end{figure*}

\begin{figure*}[h!]
\centering
\includegraphics[width=\textwidth]{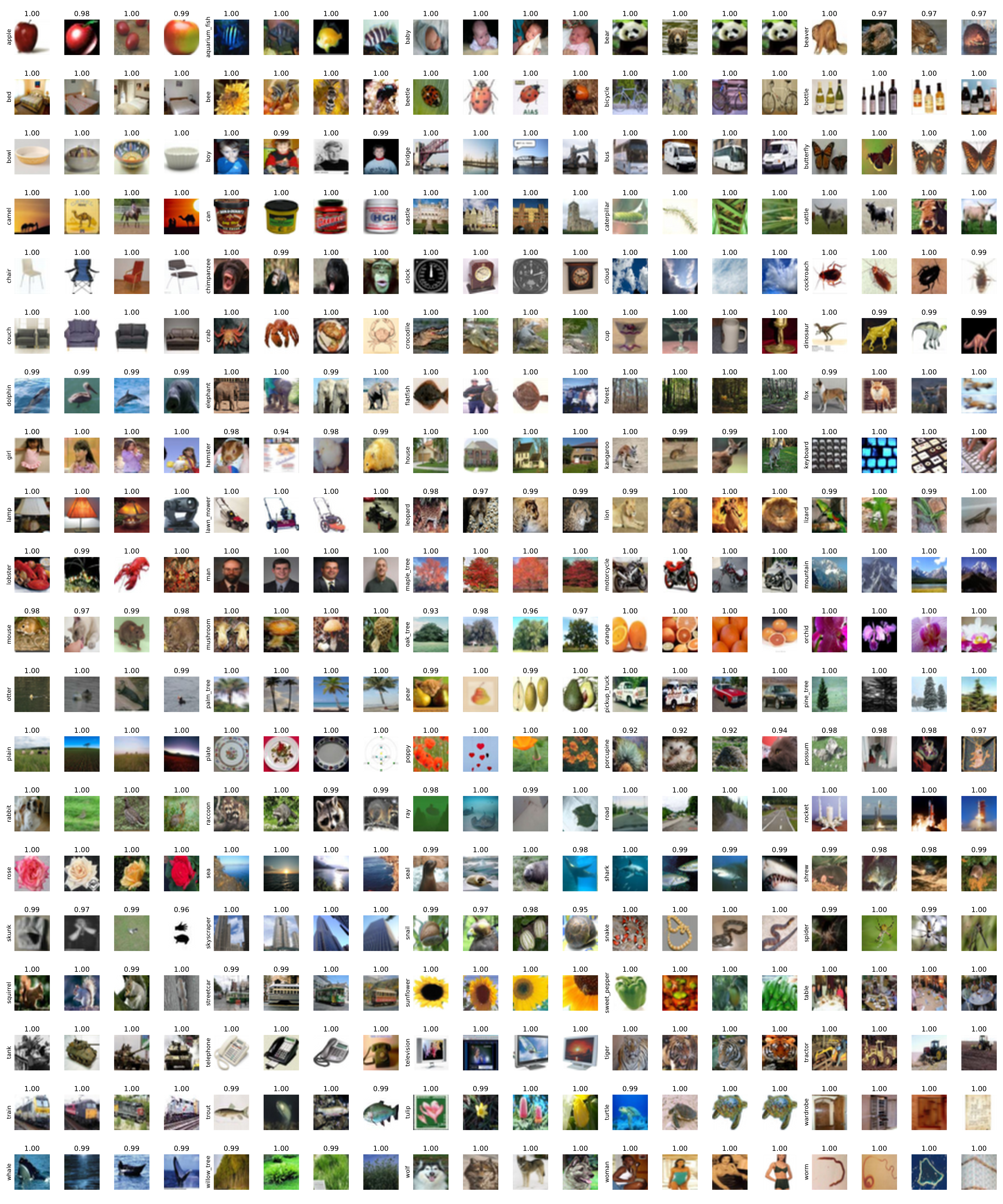}
\caption{CIFAR100: \baseod 1st iteration (up to 2.25k randomly selected samples per class). Disregarding few outlier classes, the baseline with OD thresholding is even better than \base and has a good sample quality even though clearly worse than \ours.}
\label{fig:appendix_cifar100_st_ot_1}
\end{figure*}

\begin{figure*}[h!]
\centering
\includegraphics[width=\textwidth]{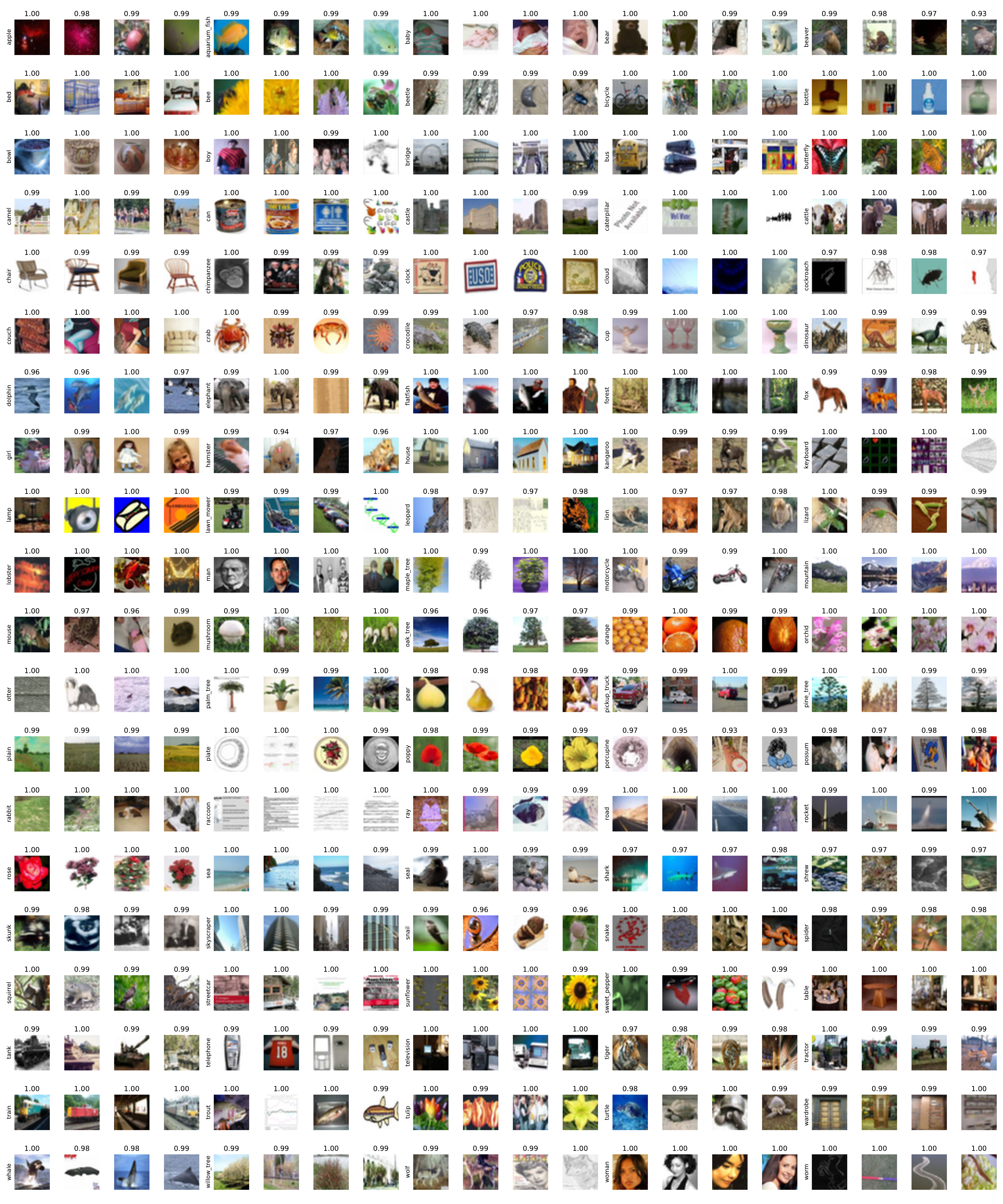}
\caption{CIFAR100: \baseod 2nd iteration (up to 4.5k randomly selected samples per class). While thresholding slightly stabilizes sample selection, without OD-aware training, the model starts to include bad samples, for example text as "raccoon" or graphics as "clock".}
\label{fig:appendix_cifar100_st_ot_2}
\end{figure*}

\begin{figure*}[h!]
\centering
\includegraphics[width=\textwidth]{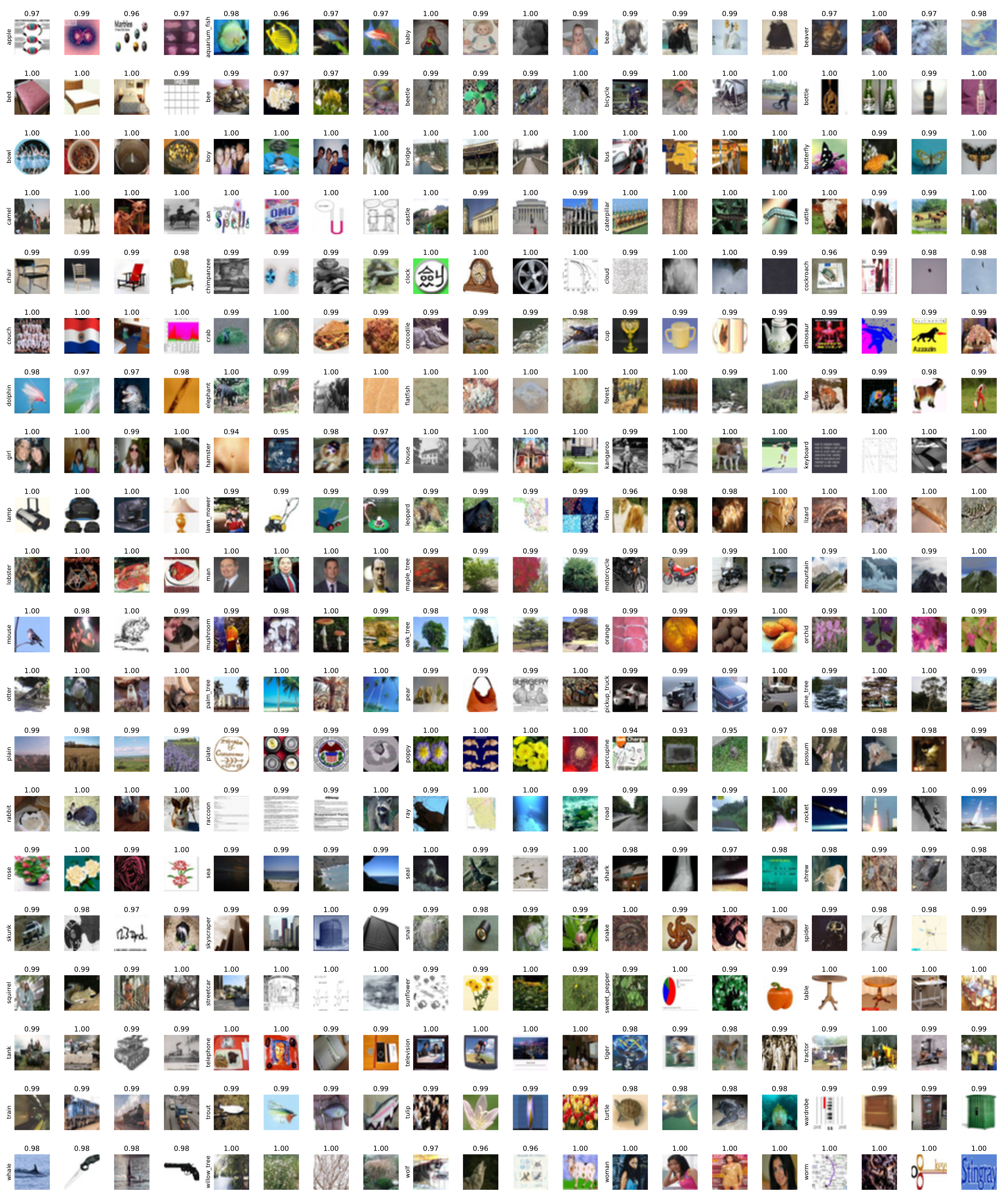}
\caption{CIFAR100: \baseod 3nd iteration (up to 6.75k randomly selected samples per class). In the third iteration, we notice more and more mistakes for a large number of classes. This highlights the challenges of self-training. Once a teacher model has learned a wrong representation, it will always pass on wrong information to the student, which results in even worse sample selection.}
\label{fig:appendix_cifar100_st_ot_3}
\end{figure*}